\documentclass[runningheads]{llncs}

\usepackage{eccv}

\usepackage{eccvabbrv}

\usepackage{graphicx}
\usepackage{booktabs}
\usepackage{multicol, multirow}
\usepackage{algorithm}
\usepackage{algpseudocode}
\usepackage[accsupp]{axessibility}  % Improves PDF readability for those with disabilities.

\usepackage[table]{xcolor}

\usepackage{hyperref}

\usepackage{orcidlink}

\usepackage{xcolor}

\definecolor{plotred}{HTML}{e15759}  % Màu đỏ của đường GACL (img-r)
\definecolor{plotblue}{HTML}{4e79a7} % Màu xanh của đường GACL (cub)

\begin{document}

\title{Spectral-Aware Analytic Class-Incremental Learning for Long-Tailed Distributions} 

\titlerunning{Geometry-Spectral Rectification (GSR)}

\author{Quyen Tran\textsuperscript{*}\inst{1} \and
Ngoc-Hai Nguyen\textsuperscript{*} \inst{2} \and
Quan Dao \inst{1} \and 
Zhuowei Li\textsuperscript{$\dagger$}\inst{1} \and
Nam Le \inst{3} \and 
Trung Le \inst{4} \and 
Dimitris Metaxas \inst{1}
}

\authorrunning{Tran et al.}

\institute{
\begin{tabular}{@{}c@{\qquad}c@{}}
\inst{1}Rutgers University \quad \inst{2}Tufts University \quad 
\inst{3}HUST \quad \inst{4}Monash University
\end{tabular}
} 
% \url{http://www.springer.com/gp/computer-science/lncs} }

\maketitle

\begingroup
\renewcommand{\thefootnote}{%
  \ifcase\value{footnote}%
  \or *%
  \or \ensuremath{\dagger}%
  \fi
}
\footnotetext[1]{Equal contribution.}
\footnotetext[2]{Work done outside Amazon.}
\endgroup

\begin{abstract}

Analytic Continual Learning (ACL) offers a computationally efficient alternative to gradient-based approaches. Recent ACL methods are based on Recursive Least Squares (RLS) and have achieved the state-of-the-art results compared to other alternatives.  However, they falter significantly in Class-Incremental Learning scenarios characterized by Long-Tailed distributions. While the ill-conditioning of the autocorrelation (Gram) matrix is a known limitation of RLS, we demonstrate that class imbalance exacerbates this issue into a distinct spectral pathology: "tail" classes suffer from severe spectral collapse, rendering their subspaces numerically indistinguishable from noise. Standard Ridge Regression ($L_2$) fails to address this effectively as it applies \textit{isotropic regularization} - a uniform penalty that is insufficient to stabilize the tail without over-shrinking the head. 
  To address this, we propose \textit{\textbf{Geometry-Spectral Rectification (GSR)}}, a theoretically grounded framework that treats long-tailed learning as a spectral regularization problem. Unlike standard isotropic regularization (Ridge) which uniformly penalizes all eigenvalues, GSR acts as an \textit{anisotropic spectral filter}, selectively inflating the collapsed eigenvalues of tail classes. 
  We construct a structured, data-dependent spectral perturbation matrix $\Delta$ that selectively inflates collapsed tail eigen-directions of the Gram matrix. 
  % Theoretical analysis proves that GSR guarantees an improved stable rank for the Gram matrix, ensuring numerical stability. 
  We show theoretically that a rectification concentrated on the tail eigenspace
increases the stable rank of the feature statistics.
  Extensive experiments show that GSR establishes a new state-of-the-art for analytic CIL, offering a superior trade-off between computational efficiency and robust generalization in long-tailed settings. Our code is \href{https://github.com/tranquyenbk173/GSR_Geometry-Spectral-Rectification}{released here.}
  \keywords{Pretrained-based Continual Learning  \and Long-tailed distribution}

%   \color{red}
% \textbf{TODO:} co len, co gang cho xong, quyet tam T.T

% - Figure: standard vs long-tailed 

% - Head/Tail acc 

% - Compare Ours vs (gauss sampling, inter-class mixup): Time, acc 

% - Add exp comments 

% - Proofread

% - Eigenspectrum of Gram matrix? (with vs w/o ours)

% - Condition number over time? (baselines vs ours) 

% - tSNE/PCA (Linear mixup vs Spherical mixup vs Gaussian sampling) 

% - Finish the main tables -- DONE
% }

% - Modify the symbols for eigenvalues (lambda), ridge coeff (xi), mixing coef (mu) -- DONE
% \color{black}

\end{abstract}

\section{Introduction}\label{sec:intro}

Analytic Continual Learning (ACL) \cite{ACIL_Zhuang_NeurIPS2022, GKEAL_Zhuang_CVPR2023, GACL_Zhuang_NeurIPS2024, DBLP:journals/corr/abs-2511-13880, mcdonnell2023ranpac} has recently emerged as a compelling "Green AI" alternative to gradient-based methods. By reformulating the learning objective into a Recursive Least Squares (RLS) problem \cite{DBLP:journals/corr/abs-2511-13880, GACL_Zhuang_NeurIPS2024, mcdonnell2023ranpac}, recent ACL methods achieve near-instantaneous updates via closed-form matrix operations, bypassing the computational burden of backpropagation, and even can achieve the upper-bound performance of Class Incremental Learning setting \cite{DBLP:journals/corr/abs-2511-13880, tran2025boosting,  momeni2025continual, tran-etal-2024-preserving}. This efficiency makes them particularly attractive for edge devices and real-time applications where resources are scarce.

However, the theoretical elegance of these methods relies on a critical, often overlooked condition: the numerical stability of the RLS update. In realistic scenarios characterized by Long-Tailed distributions, their performance degrades catastrophically (Fig.\ref{fig:baselines_imgr}). While the ill-conditioning of the autocorrelation (Gram) matrix is a known limitation of RLS, and existing works often attribute this to simple data scarcity \cite{DBLP:conf/cvpr/CuiJLSB19, cao2019learning, DBLP:conf/iclr/KangXRYGFK20, AIR_Fang_arXiv2024, DBLP:conf/aaai/ThanhLTLVN25}, we introduce a different view that the root cause is a fundamental linear algebra failure: \textit{\textbf{Spectral Collapse}}.

\begin{figure}
    \vspace{-7mm}
    \centering
    \includegraphics[width=\linewidth]{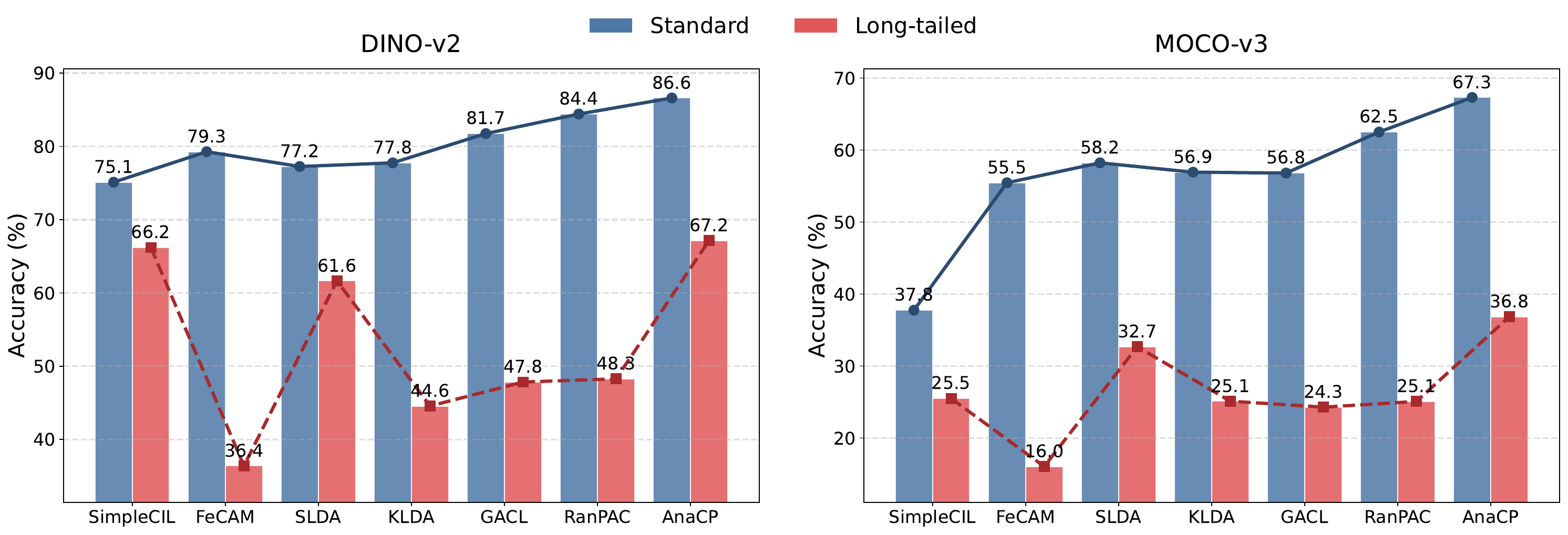}
    \vspace{-6mm}
    \caption{\textbf{Performance in} \textbf{\textit{standard and long-tailed settings}} of different ACL methods, on Split-Imagenet-R. See Fig.\ref{fig:standard_longtail_full} in Appendix \ref{additional_exp} for full results on different datasets.}
    \label{fig:baselines_imgr}
    \vspace{-6mm}
\end{figure}

In RLS-based ACL, the optimal weight solution $W^*$ is governed by the inverse of the autocorrelation (Gram) matrix $G$. We demonstrate that in imbalanced streams, $G$ suffers from severe spectral skewness. "Head" classes dominate the principal components (large eigenvalues), while "Tail" classes correspond to directions with vanishingly small variance (near-zero eigenvalues). This results in \textit{an ill-conditioned matrix} where the inversion process amplifies noise in the tail directions by orders of magnitude. Standard Ridge Regression ($\lambda I$) fails to address this because it applies a uniform penalty, which is either too weak to stabilize the tail or too strong, suppressing the head.

To bridge this gap, we propose \textbf{\textit{Geometry-Spectral Rectification (GSR)}}, a non-iterative framework that treats data imbalance as a spectral regularization problem, leveraging the statistical geometry of the feature manifold.
% Instead of learning complex hyperparameters via bilevel optimization, GSR leverages the statistical geometry of the feature manifold. 
Particularly, we introduce a class-adaptive manifold mixing strategy that synthesizes 'virtual' support features to densify the sparse tail manifolds, leading to a better-conditioned empirical covariance estimate.
Theoretically, we prove that our mixing strategy acts as a structured spectral filter, {selectively inflating the smallest eigenvalues of the Gram matrix without distorting the dominant principal components}. 
% This effectively "heals" the broken spectrum, improving the stable rank and ensuring numerical stability during inversion. 
This rebalances the spectrum of the per-task feature statistics, as reflected
by a higher stable rank.
% Notably, while techniques like Mixup \cite{zhang2018mixup, DBLP:conf/icml/VermaLBNMLB19} are traditionally used for data augmentation to improve generalization boundaries, we repurpose them for a fundamentally different goal: \textbf{\textit{Spectral Rectifica››tion.}} We show that in the context of RLS, the geometric properties of the covariance matrix are more critical than the decision boundary itself. Thus, GSR is not merely a data augmentation strategy, but a spectral regularization mechanism tailored for the recursive inverse update.
% {\color{red} [I think we should hightlight some insights and performance here.]}

\paragraph{Contributions:} In summary, our contributions are threefold:
\begin{enumerate}
    \item \textbf{Spectral Diagnosis:} We provide a rigorous analysis showing that the failure of RLS-based ACL in long-tailed settings stems from the spectral collapse phenomenon of the Gram matrix.
    \item \textbf{Geometry-Spectral Rectification (GSR):} We propose a non-parametric rectification method that respects the hyperspherical geometry of modern PTMs. 
    % Unlike Gaussian-based covariance transfer methods which introduce "blind noise" (manifold intrusion) and incur cubic computational costs ($O(D^3)$), our Geodesic Mixing strategy operates in linear time ($O(D)$) and guarantees that synthesized samples remain on the true data manifold. 
    % We prove that this strategy acts as a spectral preconditioner, explicitly improving the stable rank of the inverted matrix.
    We give conditions under which this strategy increases the stable rank of the
feature statistics.
    
    % \item \textbf{Geometry-Spectral Rectification (GSR):} Instead of heuristic replay, we propose a theoretically grounded rectification method. We prove that our specific mixing strategy acts as a spectral preconditioner, guaranteeing a reduction in the Gram matrix's condition number ($\kappa(G)$) and ensuring numerical stability for tail classes.
    \item \textbf{SOTA Performance:} Extensive experiments on benchmark datasets show that GSR significantly outperforms existing ACL baselines, bridging the gap between analytic efficiency and robust generalization in long-tailed settings.
\end{enumerate}

\section{Related Works}\label{sec:related}

\subsection{Analytic Continual Learning (ACL)}
Unlike gradient-based Continual Learning methods \cite{DBLP:conf/iclr/SahaG021, DBLP:conf/cvpr/0002ZL0SRSPDP22, tran2026optimaltransportdrivenapproachcultivating, tran2025leveraginghierarchicaltaxonomiespromptbased} that require iterative optimization and often suffer from catastrophic forgetting or high computational costs, Analytic Continual Learning \cite{ACIL_Zhuang_NeurIPS2022, GKEAL_Zhuang_CVPR2023, GACL_Zhuang_NeurIPS2024, mcdonnell2023ranpac, DBLP:journals/corr/abs-2511-13880, goswami2023fecam, Hayes_2020_CVPR_Workshops} aims to acquire new knowledge via closed-form solutions. Early works \cite{ACIL_Zhuang_NeurIPS2022, GKEAL_Zhuang_CVPR2023, GACL_Zhuang_NeurIPS2024} utilized Recursive Least Squares (RLS) to update linear classifiers incrementally. 
Recently, the proliferation of large-scale frozen pre-trained models (PTMs) has revitalized ACL \cite{mcdonnell2023ranpac, DBLP:journals/corr/abs-2511-13880, tran2025boosting, momeni2025continual}, achieving state-of-the-art results, and even each upper bound performance on balanced benchmarks.
While these methods achieve high efficiency, they implicitly assume a relatively balanced data stream. Our analysis reveals that in Long-Tailed (LT) scenarios, the standard RLS update used in SOTA methods like GACL \cite{GACL_Zhuang_NeurIPS2024}, RanPAC \cite{mcdonnell2023ranpac}, and AnaCP \cite{DBLP:journals/corr/abs-2511-13880} becomes numerically unstable. The Gram matrix for tail classes suffers from \textit{spectral collapse}, which standard isotropic regularization (Ridge) cannot fix without harming head classes. Our GSR framework specifically targets this spectral pathology.

\subsection{Long-Tailed Learning in Continual Learning Settings}

Real-world data streams frequently exhibit long-tailed distributions, where head classes dominate, and tail classes remain severely scarce. In static learning, this imbalance is commonly addressed through \textit{Re-sampling} (e.g., SMOTE \cite{DBLP:journals/jair/ChawlaBHK02}), \textit{Re-weighting} (e.g., Class-Balanced Loss \cite{DBLP:conf/cvpr/CuiJLSB19}), or \textit{Decoupling} strategies \cite{DBLP:conf/iclr/KangXRYGFK20}. In Continual Learning, these ideas have been extended to mitigate the compounded effects of imbalance and catastrophic forgetting, for example via bias-correction layers \cite{wu2019large}, cosine normalization \cite{Hou_2019_CVPR}, and balanced replay buffers \cite{kim2020imbalanced}.
% all of which aim to stabilize gradient-based optimization.
Within the ACL paradigm, Fang et al. \cite{AIR_Fang_arXiv2024} introduce inverse-frequency re-weighting into the analytic solution. However, such adjustments remain fundamentally \textit{scalar modulations}: they rescale contributions without enriching the underlying feature geometry, and cannot recover missing directions in the representation space (See results in Tables \ref{table:main_results_dino} and \ref{table:main_results_moco}). 
We argue that, in RLS-based ACL, tail-class degradation originates from data scarcity but ultimately manifests as \textit{spectral collapse}—the ill-conditioning of the autocorrelation matrix that governs the analytic solution. While existing approaches operate at the level of data counts or gradient magnitudes, they do not explicitly address this algebraic bottleneck. In contrast, we frame long-tailed ACL as a spectral conditioning problem and target it directly through explicit spectral regularization.

% \subsection{Feature Space Regularization and Manifold Mixing}
\subsection{Implicit Spectral Regularization via Geodesic Interpolation}

While standard ACL relies on explicit scalar regularization (Ridge, $\tau I$) to invert the Gram matrix, this isotropic penalty is suboptimal for long-tailed distributions. A more efficient alternative is implicit regularization via feature interpolation \cite{NIPS2019_9426, zhang2018mixup}, which is asymptotically equivalent to adding a data-dependent regularizer to the Hessian \cite{NIPS2000_ba9a56ce, 10.1162/neco.1995.7.1.108}.
However, existing strategies like Manifold Mixup \cite{DBLP:conf/icml/VermaLBNMLB19} are ill-suited for our setting. First, they employ linear interpolation to train backbone parameters, whereas our framework targets covariance conditioning for the analytic inverse on frozen PTM embeddings. Second, linear interpolation on normalized embeddings creates a geometric inconsistency: it traverses the chord rather than the arc \cite{Deng_2019_CVPR, karras2018progressive}, causing norm shrinkage and energy suppression. This prevents effective inflation of collapsed tail-class eigenvalues.
To overcome this, we propose Geodesic Interpolation (or Spherical Mixup) \cite{10.1145/325334.325242, karras2018progressive}. By strictly following the manifold geometry, our approach preserves the unit norm constraint and prevents the variance attenuation characteristic of linear mixing. This ensures the spectral perturbation $\Delta_{struct}$ effectively reconditions the tail spectrum (Theorem \ref{thm:stable_rank}) without the energy decay inherent in linear chordal traversal.

\section{Method}\label{sec:method}

\subsection{Preliminaries: Analytic Continual Learning}
This work consider Class Incremental Learning (CIL) setting, where we have a sequence of distinct tasks $t = 1, \dots, T$. At each task, the model needs to deal with the dataset $D_t = (X_t, Y_t)$, so that we have $(Z_t, Y_t)$, which are input features extracted from a pretrained model and their corresponding labels. RSL-based ACL  methods \cite{ACIL_Zhuang_NeurIPS2022, GKEAL_Zhuang_CVPR2023, GACL_Zhuang_NeurIPS2024, mcdonnell2023ranpac, DBLP:journals/corr/abs-2511-13880} aims to minimize the Ridge Regression objective:
\begin{equation}
    \mathcal{L}(W) = \|ZW - Y\|_F^2 + \tau \|W\|_F^2
\end{equation}
where $Z$ and $Y$ are all the features and their corresponding labels observed by the model so far. And $W$ is the weight matrix of the model's classification head, whose pretrained backbone is frozen. The optimal solution is given in closed form by:
\begin{equation}
    W^* = (G + \tau I)^{-1} Z^T Y
\end{equation}
where $G = Z^T Z$ is the Gram matrix. In CIL setting, $G$ and $Z^T Y$ can be updated recursively \cite{mcdonnell2023ranpac, DBLP:journals/corr/abs-2511-13880}.
The numerical stability of the inversion of $(G + \tau I)$ acts as a proxy for the bias-variance tradeoff, thereby influencing the model's generalization capability.
% The generalization capability of this analytic solution is intimately linked to the spectral properties and the conditioning of the matrix $(G + \tau I)$
% The generalization capability of this analytic solution is strictly bounded by the stability of the inversion of $(G + \tau I)$.

\subsection{The Limitation of Isotropic Regularization}
\label{sec:theory}

In this section, we analyze why standard RLS-based ACL fails in Long-Tailed (LT) scenarios through the lens of spectral geometry. We argue that the standard $L_2$ regularization creates a bad trade-off due to its \textit{isotropic} nature.

\paragraph{a. The Dilemma of Isotropic Regularization (Ridge)}
Let $G = U \Sigma U^T$ be the eigendecomposition of the Gram matrix, with eigenvalues $\lambda_1 \geq \lambda_2 \geq \dots \geq \lambda_d \geq 0$. The inversion in Eq. (2) modifies the spectrum as:
\begin{equation}
    i^{th}\text{ eigenvalue of } (G + \tau I)^{-1} = \frac{1}{\lambda_i + \tau}
\end{equation}
Here, the term $\tau I$ acts as \textit{Isotropic Regularization}. It adds a uniform spherical penalty $\tau$ to the variance in \textit{all} directions of the feature space, regardless of the data distribution.

\paragraph{b. The Spectral Pathology of Long-Tailed Data:}
In imbalanced CIL, "tail" classes span subspaces with vanishingly small variance, leading to spectral collapse ($\lambda_{tail} \approx 0$). This creates a dilemma:
\begin{itemize}
    \item \textit{Under-regularization (Small $\tau$):} If $\tau$ is small, the inverse term $(\lambda_{tail} + \tau)^{-1}$ explodes, causing the model to overfit to noise in the tail directions.
    \item \textit{Over-regularization (Large $\tau$):} If we increase $\tau$ to stabilize the tail, we inadvertently suppress the "head" classes (where $\lambda_{head} \gg \tau$). The term $(\lambda_{head} + \tau)^{-1}$ shrinks significantly, causing bias and underfitting for majority classes.
\end{itemize}
Thus, isotropic regularization cannot simultaneously stabilize the tail and preserve the head.

\subsection{Geometry-Spectral Rectification (GSR)}

To resolve this, we propose GSR, which utilizes Class-Adaptive Manifold Mixing to inject \textit{Anisotropic Information} into the collapsed dimensions.

\subsubsection{a. GSR as Structured Anisotropic Regularization}
Let $z_1, z_2$ be representations of two samples from a tail class with covariance $\Sigma_{class}$. We generate a rectified sample $\tilde{z}$ by mixing $z_1$ with $z_{2}$:
\begin{equation}
    \tilde{z} = \frac{\gamma z_1 + (1 - \gamma)z_{2}}{\| \gamma z_1 + (1 - \gamma)z_{2} \|_2}, \quad \text{where } \gamma \text{ is mixing coefficient}
    \label{mixing}
\end{equation}
% \begin{equation}
%     \mathbf{x}' = \frac{\gamma \mathbf{x}_i + (1-\gamma)\mathbf{x}_j}{\| \gamma \mathbf{x}_i + (1-\gamma)\mathbf{x}_j \|_2}, \quad \text{where } \gamma \sim \text{Beta}(\alpha, \beta)
% \end{equation}
We utilize Spherical Mixup (detailed in Appendix \ref{spherical_mix}), which introduces non-linearity via hypersphere projection, which prevents the synthetic sample from lying in the linear span of the supports. This operation prevents the synthetic samples from collapsing onto the linear convex hull, thereby implicitly regularizing the covariance rank.
In addition, from the perspective of Vicinal Risk Minimization \cite{NIPS2000_ba9a56ce}, training on mixed samples is asymptotically equivalent to training with a data-dependent regularizer. The second moment of the mixed data transforms as:
\begin{equation}
    \mathbb{E}[\tilde{x}\tilde{x}^T] \approx G_{original} + \Delta
\end{equation}
Crucially, unlike $\tau I$, the term $\Delta$ is proportional to the class covariance $\Sigma_{class}$. This implies that GSR acts as an \textit{Anisotropic Spectral Regularizer}:
\begin{enumerate}
    \item \textit{Directional Alignment:} $\Delta$ adds "mass" (variance) specifically along the principal components of the class manifold, not in orthogonal null-spaces.
    \item \textit{Tail Rectification:} For a tail class, this selectively inflates the eigenvalues $\lambda_{tail}$ ("thickening" the manifold) without needing a large global $\gamma$.
\end{enumerate}

\subsubsection{b. Theoretical Guarantee}

Define $\tilde{G} = G + \Delta$ is update of $G$ after pertubation. 
% We formally show that this strategy improves numerical stability through the increase of \emph{stable rank}, which supports the success of GSR in recovering the intrinsic manifold structure of the data, while reducing overfitting driven by heavy-tail behavior in the spectrum. 
We formally relate this strategy to an increase of the stable rank, a measure
of how evenly the spectral energy is spread over directions. Formally, \emph{stable rank} for a matrix $G$ is
${\text{sr}(G)=\frac{\sum_{i=1}^d \lambda_i^2}{\lambda_1^2}}.$
\begin{theorem}[Quantitative stable-rank improvement under tail-mixup]\label{thm:stable_rank}
Fix an integer $r\in\{1,\ldots,d-1\}$ and let $U_h\in\mathbb{R}^{d\times r}$ collect the top-$r$
orthonormal eigenvectors of $G$, and $U_t\in\mathbb{R}^{d\times (d-r)}$ collect the remaining ones. Define the head/tail projectors $P_h=U_hU_h^\top, P_t=U_tU_t^\top$. Assume the following blockwise quantitative bounds:
$$
\|U_h^\top \Delta U_h\|_2 \le \varepsilon,\quad
\|U_h^\top \Delta U_t\|_2 \le \eta, \quad \|U_t^\top \Delta U_t\|_2 \le \bar m,\quad
m \le \lambda_{\min}(U_t^\top \Delta U_t),
$$
for some parameters $\varepsilon\ge 0$, $m>0$, $\bar m\ge m$, and $\eta\ge 0$. Define
\[
S_\tau \;\triangleq\; \|G+\tau I\|_F^2
\;=\;\sum_{i=1}^d (\lambda_i+\tau)^2,
\]
and
\[
\Lambda_+ \;\triangleq\;
\frac{(\lambda_1+\varepsilon)+(\lambda_{r+1}+\bar m)
+\sqrt{\big((\lambda_1+\varepsilon)-(\lambda_{r+1}+\bar m)\big)^2+4\eta^2}}{2}.
\]
Then for any ridge parameter $\tau\ge 0$,
\begin{equation}\label{eq:sr_lower_bound}
\text{sr}(\widetilde G+\tau I)
\;\ge\;
\frac{S_\tau + (d-r)m^2}{(\Lambda_+ + \tau)^2}.
\end{equation}
Moreover, a sufficient
condition for \emph{strict} improvement is
\begin{equation}\label{eq:sr_strict_condition_general}
(d-r)m^2
\;>\;
S_\tau\left(\frac{(\Lambda_+ + \tau)^2}{(\lambda_1+\tau)^2}-1\right).
\end{equation}
\end{theorem}

\begin{corollary}\label{cor:ideal_tail_only}
Assume $ \Delta\succ 0$ satisfies $U_h^\top \Delta U_h=0$ and $U_h^\top \Delta U_t=0$
(i.e., $\Delta$ is supported entirely in the tail eigenspace of $G$), and the top
eigenvalue remains head-dominated: $\lambda_1 \ge \lambda_{r+1} + \|U_t^\top \Delta U_t\|_2$.
Then 
$$\text{sr}(\tilde{G}) > \text{sr}(G)$$
\end{corollary}

We give the full proof for all these theorems in the Appendix \ref{sec:appendix_proof}. 
% Theorem~\ref{thm:stable_rank} shows that GSR \emph{strictly increases} $\mathrm{sr}(\widetilde G+\tau I)$ by injecting structured mass into the tail eigenspace (via $\Delta$), while controlling the increase of the top eigenvalue through blockwise bounds. In other words, GSR does not blindly add isotropic noise; it \emph{densifies} the collapsed tail spectrum in a targeted way.
Theorem~\ref{thm:stable_rank} shows that a perturbation $\boldsymbol{\Delta}$ that injects mass into the tail eigenspace strictly increases the stable rank, provided the increase of the top eigenvalue is controlled through the blockwise bounds. The theorem holds for any positive semi-definite $\mathbf{G}$ and $\boldsymbol{\Delta}$; Section~4.3 examines it empirically on the mean-centred statistics of each task.
% This has two concrete consequences:

% \begin{itemize}
% \item \textbf{Reduced sensitivity to tail noise and stream perturbations.}
% Spectral collapse creates directions with near-zero variance that are numerically indistinguishable from noise; any small disturbance (domain shift, outliers, label noise) can dominate these directions after inversion.
% By increasing $\mathrm{sr}(\widetilde G+\tau I)$, GSR prevents the solution from being supported on an extremely low-dimensional subspace, and makes the analytic estimate less brittle in the tail regime.

% \item \textbf{Better retention and generalization for tail concepts.}
% A higher stable rank means more directions remain ``alive'' for representing minority classes.
% This improves the representational capacity available to tail classes under scarcity, reducing the tendency to ignore them (near-random tail accuracy) and mitigating forgetting across tasks.
% \end{itemize}

% \vspace{0.25em}
% \noindent\textbf{Connection to anisotropic regularization.}
% Standard Ridge adds $\tau I$ uniformly across all directions, which cannot selectively repair collapsed tail dimensions without over-shrinking head directions.
% In contrast, GSR acts as an \emph{anisotropic spectral rectifier}: it selectively inflates the collapsed eigenvalues of tail subspaces through $\Delta$, increasing stable rank while preserving head structure.
% This resolves the stability--plasticity tension at the spectral level and explains the consistent empirical gains of GSR in long-tailed analytic CIL.

Furthermore, GSR acts as an \textit{Anisotropic Spectral Regularizer}. Unlike standard $L_2$ regularization (Ridge), which adds a constant damping factor $\tau I$ uniformly across all directions (penalizing head and tail classes equally), GSR selectively "inflates" only the collapsed eigenvalues of tail subspaces via $\Delta$. This allows the model to retain high precision on head classes (where eigenvalues are already large) while stabilizing the learning of tail classes, effectively resolving the stability-plasticity dilemma at the spectral level.

One might be concerned that injecting $\Delta$ into the null-space could amplify noise. However, in the long-tailed regime, the zero eigenvalues in tail classes predominantly stem from sample scarcity (feature collapse) rather than the absence of semantic information. Crucially, unlike standard Tikhonov regularization (Ridge Regression), which adds isotropic noise (white noise), our injected $\Delta$ is anisotropic. 
% \textcolor{red}{By leveraging statistics from head classes, we selectively restore variance along meaningful semantic directions while suppressing actual noise.} 
By leveraging the latent geometric structure within the scarce samples themselves, we selectively inflate the collapsed spectral components of tail classes, ensuring the restored geometry reflects plausible semantic variations rather than arbitrary disturbances.
% This ensures that the restored geometry reflects plausible intra-class variations rather than arbitrary disturbances 
% \subsubsection{Implication of Theorem 1}
% The inequality derived in Eq. (\ref{eq:inequality_condition}) provides a theoretical guarantee for the effectiveness of GSR.
% \begin{itemize}
% \item \textbf{Restoring Feature Diversity:} The condition $\frac{\delta_D}{\lambda_D} \gg \frac{\delta_1}{\lambda_1}$ implies that GSR injects variance precisely where it is needed most—the collapsed tail dimensions. This counteracts the "Spectral Collapse" phenomenon, ensuring that tail classes acquire a distinct feature subspace rather than being treated as noise.
% \item \textbf{Optimization Stability:} A reduced condition number $\kappa(\tilde{G})$ implies a more spherical loss landscape. In the context of Graph Convolutional Networks (GCNs), this improves the convergence rate of gradient descent and prevents the gradients of dominant head classes from overshadowing those of tail classes during backpropagation.
% \item \textbf{Robustness to Over-smoothing:} By fortifying the smallest eigenvalues, GSR prevents the feature representations from degenerating into a single point (the over-smoothing problem) as the network depth increases, thereby preserving the discriminative power of tail nodes.
% \end{itemize}

\subsubsection{c. Our technical method}
% \paragraph{Adaptive Mixing Strategy.}
% To operationalize Theorem \ref{thm:stable_rank}, we define the mixing intensity $\alpha_c$ for class $c$ inversely proportional to its spectral health:
% \begin{equation}
%     \alpha_c = \alpha_{base} + (1 - \alpha_{base}) \cdot e^{-\xi \cdot N_c}
%     \label{alpha_mix}
% \end{equation}
% where $N_c$ is the sample count of class $c$, and $\xi$ is the decay rate. This ensures that tail classes undergo intensive manifold mixing (high variance injection), while head classes preserve their original precise feature statistics.

\paragraph{Class-Adaptive Rectification Budget.}
To operationalize Theorem~\ref{thm:stable_rank}, we let the amount of rectification that a class
receives depend on its spectral health, which we measure by its sample count.
For a task $t$ with classes $\mathcal{C}_t$, let $N_c$ be the number of real
samples of class $c$ and $T_t = \min\!\big(\max_{c \in \mathcal{C}_t} N_c,\; T_{\max}\big)$.
Class $c$ contributes
\begin{equation}
    M_c = \max\!\big(T_t - N_c,\; 0\big)
    \label{eq:budget}
\end{equation}
synthesized samples to the rectification term. Tail classes, whose subspaces are
collapsed, therefore receive intensive manifold mixing, while head classes with
$N_c \ge T_t$ are left untouched and preserve their original feature statistics.
This is in contrast to standard mixup, which mixes every sample equally often and
thus allocates the synthesized samples in proportion to $N_c$.

% \subsubsection{Rectified Gram Matrix Update}
% For an incoming batch, let $x_i$ be a sample of class $y_i$. We retrieve a stored prototype or random sample $x_{ref}$ from the buffer belonging to the same class. We generate a mixed feature $\tilde{x}_i$ using Eq. (4). Crucially, we do not just augment data; we modify the Gram matrix update. The rectified update becomes:
% \begin{equation}
%     \tilde{G}_t = G_{t-1} + \sum_{i \in \mathcal{B}} \tilde{x}_i \tilde{x}_i^T
% \end{equation}

\paragraph{Rectified Gram Matrix Update via Tail Mixup.}
% Instead of relying on external memory buffers or stored prototypes, w
We leverage the scarce real samples available for tail classes to induce spectral rectification directly.
% \quad \text{where } \lambda \sim \text{Beta}(\alpha, \beta)
Let $\mathcal{X}_{real} = \{x_i\}_{i=1}^{N_c}$ be the set of current real samples, and $\mathcal{Z}_{real} = \{z_i\}_{i=1}^{N_c}$ be the set of their corresponding representations for a specific class $c$. 

Due to the scarcity of samples ($N_c \ll D$, $D$ is the feature dimension), the empirical covariance of $\mathcal{Z}_{real}$ is rank-deficient. To counteract this phenomenon, we synthesize a set of augmented samples $\mathcal{Z}_{aug}$ by interpolating along the geodesic paths between real data pairs. Unlike standard linear mixup, which shrinks the feature norm, our spherical formulation preserves the hyper-spherical geometry:
\vspace{-4mm}

\begin{algorithm}[t]
\caption{Geometry-Spectral Rectification (GSR) for RLS-based ACL}
\label{alg:gsr}
\begin{algorithmic}[1]
\Require Initial Gram matrix $G_0 = 0$, cross-correlation $Q_0 = 0$, regularization $\tau$, rectification weight $\beta$, budget cap $T_{\max}$, Beta parameter $a$.
\Require Data coming in a sequence of tasks $\mathcal{T} = \{\mathcal{D}_1, \dots, \mathcal{D}_T\}$.
\Ensure Class weights $W_{final}$.
\For{task $t = 1$ to $T$}
    \State Receive data $\mathcal{D}_t = \{(x_i, y_i)\}_{i=1}^{N_t}$.
    \State Extract features $Z_t = f_\theta(X_t) \in \mathbb{R}^{N_t \times D}$, where $D$ is the feature dimension.
    \State Initialize augmented buffer $Z_{aug} = \emptyset$.
    \State Compute $T_t = \min\big(\max_{c \in \mathcal{C}_t} N_c,\; T_{\max}\big)$.
    \For{class $c \in \mathcal{C}_t$}
        \State \textbf{\textit{// 1. Spectral Diagnosis \& Adaptive Budget}}
        \State Calculate sample count $N_c$.
        \State Compute rectification budget $M_c = \max(T_t - N_c,\, 0)$. \Comment{Using Eq.~\ref{eq:budget}}
        \State \textbf{\textit{// 2. Geometry-Spectral Rectification}}
        \For{$m = 1$ to $M_c$}
            \State Draw a pair $(z_i, z_j)$ uniformly at random from $Z_{t,c}$.
            \State Sample $\gamma \sim \text{Beta}(a, a)$.
            \State Synthesize spherical mixup sample \Comment{Using Eq.~\ref{eq:spherical_mixup}}
            \Statex \hspace{5.2em} $\tilde{z} = \dfrac{\gamma z_i + (1-\gamma) z_j}{\|\gamma z_i + (1-\gamma) z_j\|_2}$
            \State $Z_{aug} \leftarrow Z_{aug} \cup \{\tilde{z}\}$, with label $c$.
        \EndFor
    \EndFor
    \State \textbf{\textit{// 3. Block-wise Update (Dual Stream)}}
    \State Update $Q_t \leftarrow Q_{t-1} + Z_t^{T} Y_t + \beta Z_{aug}^{T} Y_{aug}$.
    \State Update $G_t \leftarrow G_{t-1} + Z_t^{T} Z_t + \beta Z_{aug}^{T} Z_{aug}$. \Comment{Using Eq.~\ref{unify_G}}
    \State \textbf{\textit{// 4. Analytic Solution}}
    \State Compute weights: $W_t = (G_t + \tau I)^{-1} Q_t$.
\EndFor
\State \Return $W_{final} = W_T$
\end{algorithmic}
\end{algorithm}

\begin{equation}
    \tilde{z} = \frac{\gamma z_i + (1-\gamma) z_j}{\| \gamma z_i + (1-\gamma) z_j \|_2},
    \quad \text{where } z_i, z_j \in \mathcal{Z}_{real} \text{, and } \gamma \sim \text{Beta}(a, a)
    \label{eq:spherical_mixup}
\end{equation}
Here, the $M_c$ mixing pairs $(z_i, z_j)$ of class $c$ are drawn uniformly at
random from its real samples, $\gamma$ is drawn from a symmetric Beta distribution
(we use $a = 2$ throughout), and the normalization keeps $\tilde{z}$ on the unit
hyper-sphere, avoiding the norm shrinkage of linear mixup.

Crucially, to ensure both fidelity to the true distribution and structural robustness, we utilize a \textit{Dual-Stream Update} strategy. The Gram matrix is updated using the union of both real and augmented data streams. The update rule becomes:
% \begin{equation}
%     \tilde{G}_t = G_{t-1} + \underbrace{\sum_{z \in \mathcal{Z}_{real}} z z^T}_{\text{Fidelity Term}} + \underbrace{\beta \sum_{\tilde{z} \in \mathcal{Z}_{aug}} \tilde{z} \tilde{z}^T}_{\text{Rectification Term}}
%     \label{unify_G}
% \end{equation}
% \begin{equation}
%     \tilde{G}_t = G_{t-1} + \sum_{c \in \mathcal{C}_t} \Bigg[
%         \underbrace{\sum_{i=1}^{N_c} z_{c,i}\, z_{c,i}^{T}}_{\text{Fidelity Term}}
%         + \underbrace{\beta \sum_{m=1}^{M_c} \tilde{z}_{c,m}\, \tilde{z}_{c,m}^{T}}_{\text{Rectification Term}}
%     \Bigg],
%     \qquad 
%     \\
%     \tilde{Q}_t = Q_{t-1} + \sum_{c \in \mathcal{C}_t} \Bigg[
%         \sum_{i=1}^{N_c} z_{c,i}\, y_c^{T} + \beta \sum_{m=1}^{M_c} \tilde{z}_{c,m}\, y_c^{T}
%     \Bigg]
%     \label{unify_G}
% \end{equation}
\begin{equation}
\begin{aligned}
    \tilde{G}_t &= G_{t-1} + \sum_{c \in \mathcal{C}_t} \Bigg[
        \underbrace{\sum_{i=1}^{N_c} z_{c,i}\, z_{c,i}^{T}}_{\text{Fidelity Term}}
        + \underbrace{\beta \sum_{m=1}^{M_c} \tilde{z}_{c,m}\, \tilde{z}_{c,m}^{T}}_{\text{Rectification Term}}
    \Bigg], \\[4pt]
    \tilde{Q}_t &= Q_{t-1} + \sum_{c \in \mathcal{C}_t} \Bigg[
        \sum_{i=1}^{N_c} z_{c,i}\, y_c^{T}
        + \beta \sum_{m=1}^{M_c} \tilde{z}_{c,m}\, y_c^{T}
    \Bigg]
\end{aligned}
\label{unify_G}
\end{equation}
where $\tilde{z}_{c,1}, \dots, \tilde{z}_{c,M_c}$ are the spherical mixup samples of
class $c$ (Eq.~\ref{eq:spherical_mixup}) and $y_c$ is its one-hot label. In
expectation, the rectification term of class $c$ equals
$\beta\, M_c\, \tilde{\Sigma}_c$ with
$\tilde{\Sigma}_c = \mathbb{E}\big[\tilde{z}\tilde{z}^{T}\big]$ the second moment of
its mixed features: the budget $M_c$ of Eq.~(\ref{eq:budget}) directly sets how
much spectral mass is injected into the subspace of class $c$, from
$\beta\,(T_t - N_c)$ for a tail class down to zero for a head class.
% where $\beta$ is a balancing coefficient. 
Particularly:
\begin{itemize}
    \item \textit{Fidelity Term:} ensures the model learns the precise feature statistics from the real data.
    % \item \textit{Rectification Term:} injects the necessary variance into the null-space of the tail class manifold, preventing the spectral collapse described in Theorem \ref{thm:stable_rank}.
     \item \textit{Rectification Term} injects variance along the geodesic arcs into the subspace of the tail class, which Theorem \ref{thm:stable_rank} relates to a higher stable rank.
     % into the null-space of the tail class manifold, preventing the spectral collapse described in Theorem \ref{thm:stable_rank}.
\end{itemize}
% By combining both, GSR preserves the original signal while artificially "thickening" the manifold to ensure numerical stability during the inversion of $G$.
By combining both, GSR preserves the original semantic identity while enlarging the contribution of tail classes to the statistics used by RLS-based ACL methods in long-tailed settings. The complete procedure of our technical method is summarized in Algorithm \ref{alg:gsr}.
\vspace{-5mm}
\section{Experiments}\label{sec:exp}
\vspace{-4mm}
\subsection{Experimental Setup}

\noindent \textbf{Datasets and Protocols.} We evaluate GSR on four benchmark datasets widely used in Class-Incremental Learning (CIL), but in Long-Tailed setting: {Split-CIFAR-100}, {Split-ImageNet-R}, {Split-Tiny-ImageNet}, and {Split-CUB-200}. Further details regarding the long-tailed datasets can be found in Appendix \ref{exp_detail}.
% {\color{red} I think we should give more detail about the long tail setup here? (e.g. how we sampling long tail class, why it suitable)}

\noindent \textbf{Baselines.} We compare GSR against a wide range of representative ACL baselines, categorized into two groups: \textit{(i) RLS-based methods}, including RanPAC \cite{mcdonnell2023ranpac}, AnaCP \cite{DBLP:journals/corr/abs-2511-13880}, GACL \cite{GACL_Zhuang_NeurIPS2024}, and AIR \cite{AIR_Fang_arXiv2024}; and \textit{(ii) Covariance and Prototype-based methods}, including FECAM \cite{goswami2023fecam}, SLDA \cite{Hayes_2020_CVPR_Workshops}, KLDA \cite{momeni2025continual}, and SimpleCIL \cite{DBLP:journals/corr/abs-2210-04428}.

\noindent \textbf{Evaluation Metrics.} We report the performance using: Last Accuracy ($A_{last}$) and  Average Accuracy ($A_{avg}$).
% \begin{itemize}
%     \item \textbf{Last Accuracy ($A_{last}$):} The overall test accuracy after learning the final task $T$.
%     \item \textbf{Average Accuracy ($A_{avg}$):} The mean of test accuracies calculated after each task step, reflecting the stability of the learning curve.
%     \item \textbf{Forgetting ($F$):} The average performance drop of old classes after learning new tasks.
% \end{itemize}

All experiments use DINO-v2 or MoCo-v3 as pretrained backbones and are averaged over 3 random seeds to ensure statistical significance. 
\noindent See further details in Appendix \ref{exp_detail}.
\vspace{-3mm}
\subsection{Experimental results}
\subsubsection{Performance analysis on different backbones}
% \paragraph{\color{red} a. The "Strong Backbone" Paradox and the Dominance of SimpleCIL: I think this header is too fancy like a marketing style :D, I dont like it} ---> đmm, hihi
\paragraph{a. Performance Reversal: Simple Baselines Outperform RLS-based "SOTA methods" in Long-tailed Settings}
Observing Table \ref{table:main_results_dino} and Figure \ref{fig:baselines_imgr}, we identify an intriguing phenomenon: in a long-tailed setting with a robust backbone like DINO-v2, the simple method SimpleCIL achieves remarkably high performance (80.42\% on CIFAR-100), far surpassing RLS-based ACL methods such as GACL (48.78\%) and RanPAC (49.30\%).

Typically, in balanced settings, RLS leverages the Gram matrix to utilize second-order statistics for feature decorrelation and minimizes the global Mean Squared Error (MSE), yielding precise decision boundaries. However, this strength becomes a liability in long-tailed scenarios. In the high-dimensional feature space of DINO-v2, combined with data scarcity, the covariance estimation becomes ill-conditioned. This leads to severe overfitting, causing performance to degrade significantly compared to the simple baseline.
% Observing Table \ref{table:main_results_dino} and Figure \ref{fig:baselines_imgr}, we identify an intriguing phenomenon: with a robust backbone like DinoV2, the simple method SimpleCIL achieves remarkably high performance (80.42\% on CIFAR-100), far surpassing complex Analytic Learning methods such as GACL (48.78\%) and RanPAC (49.30\%). The reason is that GACL and RanPAC rely heavily on estimating the inverse covariance matrix. In the high-dimensional feature space of DinoV2, combined with limited data availability (typical of the continual learning setting), this estimation becomes ill-conditioned. This leads to severe overfitting, causing performance to fall below average.

\paragraph{b. Effectiveness of GSR on DINO-v2 backbone:}
Our method acts as a powerful regularizer, effectively addressing the critical drawback mentioned above:
% \begin{itemize}
%     \item \textit{Revitalizing Collapsed Algorithms:} When applying our method (+Ours), the performance of GACL and RanPAC skyrockets (e.g., GACL increases from 48.78\% $\rightarrow$ 65.81\% on CIFAR-100). This demonstrates that our method stabilizes parameter estimation within high-dimensional spaces.
%     \item \textit{Establishing a New SOTA with AnaCP:} More importantly, although SimpleCIL is a formidable competitor, the combination of AnaCP and our method (AnaCP + Ours) achieves 82.51\%, officially surpassing SimpleCIL to become the State-of-the-Art (SOTA). This confirms that to fully exploit the power of DinoV2, a strong Analytic algorithm (AnaCP) supported by high-quality augmented data from our method is essential.
% \end{itemize}
\begin{itemize}
    \item \textit{Revitalizing RLS-based Algorithms:} Integrating our method (+Ours) yields substantial performance gains for GACL and RanPAC (e.g., GACL improves from 48.78\% $\rightarrow$ 65.51\% on Split-CIFAR-100). This demonstrates that our approach effectively stabilizes parameter estimation in high-dimensional spaces, mitigating the overfitting issue observed in the original baselines.
    
    \item \textit{Achieving State-of-the-Art with AnaCP:} While SimpleCIL serves as a highly competitive baseline, the combination of AnaCP and our method (AnaCP + Ours) reaches 82.65\%, surpassing SimpleCIL to establish a new State-of-the-Art. This confirms that to fully exploit the power of DinoV2, a strong Analytic algorithm (AnaCP) supported by high-quality augmented data from our method is essential.
    % This confirms that fully exploiting the potential of DinoV2 requires a robust Analytic algorithm (AnaCP) complemented by the high-quality data augmentation provided by our method.
    \item \textit{Our advantage over long-tail-aware RLS methods:} Among the baselines, AIR stands out as the only RLS-based ACL method explicitly designed for long-tailed settings by introducing inverse-frequency re-weighting into the Gram matrix. However, we argue that such adjustments remain fundamentally scalar modulations. They merely rescale the contribution of existing features without enriching the underlying feature geometry, failing to recover missing directions in the representation space. Consequently, as shown in the tables, AIR underperforms AnaCP + Ours on all datasets and both backbones. On MoCo-v3 it remains ahead of GACL + Ours, which indicates that re-weighting is a strong correction for a plain ridge classifier.
\end{itemize}
\paragraph{c. Consistency on MoCo-v3:}
The improvement trend remains consistent on MoCo-v3 (Table \ref{table:main_results_moco}). AnaCP + Ours continues to lead with a significant margin over the baseline (increasing from 57.38\% to 68.57\% on Split-CIFAR-100). This highlights the generalization capability of the proposed method across different feature extractor architectures.

% Tables \ref{table:main_results_dino}, \ref{table:main_results_moco}

\begin{table}[!ht]
% \vspace{-3mm}
	\centering
    \scriptsize
    \caption{\textbf{Overall performance comparison, using DINO-v2 as PTM.} The best result is shown in \textbf{bold}, and the second-best is \underline{underlined}.}
    \vspace{-1mm}
	\label{table:main_results_dino}
    \smallskip
      \renewcommand\arraystretch{1.2}
    \small{
	\resizebox{\textwidth}{!}{
	\begin{tabular}{l cc cc cc cc}
	   \hline
        
        \multirow{2}{*}{Method} & \multicolumn{2}{c}{Split CIFAR-100} & \multicolumn{2}{c}{Split ImageNet-R} & \multicolumn{2}{c}{Split-Tiny-Imagenet} & \multicolumn{2}{c}{Split CUB-200} \\
        \cmidrule(lr){2-3} \cmidrule(lr){4-5} \cmidrule(lr){6-7} \cmidrule(lr){8-9}
        & {$A_{last}$} ($\uparrow$)  & {$A_{avg}$} ($\uparrow$) & {$A_{last}$} ($\uparrow$)  & {$A_{avg}$} ($\uparrow$) & {$A_{last}$} ($\uparrow$)  & {$A_{avg}$} ($\uparrow$) & {$A_{last}$} ($\uparrow$)  & {$A_{avg}$} ($\uparrow$) \\
        \hline
    %    \multirow{6}*{\tabincell{c}{Sup-21K}} 
       % Join Linear probe \\
       % Join fine-tuning \\
       SLDA & $72.28_{\pm 0.67}$ &$76.06_{\pm 1.30}$ & $61.65_{\pm 0.32}$ &$64.60_{\pm 1.08}$ & $71.19_{\pm 0.43}$ &$73.90_{\pm 1.49}$ & $81.11_{\pm 0.82}$ &$86.21_{\pm 0.80}$\\ 
       FeCAM & $33.35_{\pm 0.48}$ &$37.97_{\pm 1.71}$  & $36.41_{\pm 1.29}$ &$42.89_{\pm 2.77}$ & $35.28_{\pm 0.27}$ &$39.68_{\pm 1.07}$ & $70.55_{\pm 0.54}$ &$79.31_{\pm 0.90}$\\
       SimpleCIL & $\underline{80.42}_{\pm 0.33}$ &$\underline{87.21}_{\pm 0.40}$ & $66.20_{\pm 1.28}$ &$\underline{73.73}_{\pm 0.60} $ & $\underline{78.21}_{\pm 0.13}$ &$\underline{83.60}_{\pm 1.03}$ & $82.07_{\pm 0.29}$ &$87.84_{\pm 0.83}$\\
       KLDA & $49.30_{\pm 1.37}$ &$55.38_{\pm 2.72}$ & $44.57_{\pm 1.74}$ &$51.26_{\pm 1.97}$ & $52.31_{\pm 0.78}$ &$54.84_{\pm 1.32}$ & $78.16_{\pm 0.98}$ &$85.98_{\pm 1.26}$\\
       AIR & $64.87_{\pm 0.93}$ & $69.92_{\pm 2.48}$ & $54.51_{\pm 1.35}$ & $58.87_{\pm 1.29}$ & $69.19_{\pm 0.45}$ & $69.89_{\pm 1.41}$ & $80.74_{\pm 1.04}$ & $87.39_{\pm 1.38}$ \\
       \midrule
       \midrule
       
       GACL & $48.78_{\pm 1.32}$ &$58.84_{\pm 2.97}$ & $47.84_{\pm 2.18}$ &$51.53_{\pm 1.34}$ & $42.98_{\pm 0.37}$ &$53.65_{\pm 1.32}$ & $80.77_{\pm 1.07}$ &$87.20_{\pm 1.17}$\\
       \rowcolor{gray!20} ${\quad +Ours}$ & $65.51_{\pm 1.25}$ & $71.79_{\pm 2.69}$ & $61.58_{\pm 0.58}$ &$65.90_{\pm 0.72}$ & $68.88_{\pm 0.29}$ &$70.70_{\pm 1.27}$ & $82.46_{\pm 0.58}$ &$87.98_{\pm 1.00}$\\
       % GACL$_{+Ours}$ & $65.81_{\pm 1.25}$ & & $61.58_{\pm 0.58}$ & & $68.88_{\pm 0.29}$ & & $82.46_{\pm 0.58}$\\
       \midrule
       RanPAC & $49.30_{\pm 1.27}$ & $59.03_{\pm 1.03}$ & $48.28_{\pm 1.67}$  & $52.94_{\pm 2.97}$ & $43.62_{\pm 0.36}$ & $54.84_{\pm 1.23}$ & $81.11_{\pm 0.89}$ & $87.52_{\pm 1.25}$\\
       \rowcolor{gray!20} ${\quad +Ours}$ & $66.68_{\pm 1.25}$ & $71.98_{\pm 2.65}$ & $62.08_{\pm 0.62}$ & $66.65_{\pm 1.07}$ & $69.57_{\pm 0.29}$ & $71.45_{\pm 1.25}$ & $83.03_{\pm 0.62}$ & $88.64_{\pm 1.02}$\\
       
       \midrule
       AnaCP & $72.90_{\pm 0.01}$ &$78.54_{\pm 1.72}$ & $\underline{67.15}_{\pm 0.19}$ &$71.86_{\pm 0.11}$ &  $72.97_{\pm 0.26}$ &$76.30_{\pm 1.23}$ & $\underline{83.56}_{\pm 0.66}$ &$\underline{89.07}_{\pm 1.09}$\\

        \rowcolor{gray!20} ${\quad +Ours}$  & $\textbf{82.65}_{\pm 0.46}$ &$\textbf{88.41}_{\pm 0.65}$ & $\textbf{72.80}_{\pm 0.84}$ &$\textbf{78.57}_{\pm 0.45}$ & $\textbf{79.92}_{\pm 0.05}$ &$\textbf{84.88}_{\pm 0.83}$ & $\textbf{85.82}_{\pm 0.22}$ & $\textbf{90.58}_{\pm 0.64}$ \\
       
       \midrule
	\end{tabular}
    }}
    \vspace{-1mm}
\end{table}

\begin{table}[!ht]
	\centering
    \caption{\textbf{Overall performance comparison, using MoCo-v3 as PTM.} The best result is shown in \textbf{bold}, and
the second-best is \underline{underlined}.}
	\label{table:main_results_moco}
    \vspace{-0.1cm}
    
    % \vspace{-0.32cm}
    %\cite{smith2022coda} uses self-supervised pretraining on ImageNet-21K with supervised fine-tuning on ImageNet-1K, and different data splits for ImageNet-R. Here we reproduce its results with the same settings.
    
      %\vspace{-0.1cm}
	\smallskip
      \renewcommand\arraystretch{1.2}
    \small{
	\resizebox{\textwidth}{!}{ 
	\begin{tabular}{l cc cc cc cc}
	 \hline
        
        \multirow{2}{*}{Method} & \multicolumn{2}{c}{Split CIFAR-100} & \multicolumn{2}{c}{Split ImageNet-R} & \multicolumn{2}{c}{Split-Tiny-Imagenet} & \multicolumn{2}{c}{Split CUB-200} \\
        \cmidrule(lr){2-3} \cmidrule(lr){4-5} \cmidrule(lr){6-7} \cmidrule(lr){8-9}
        & {$A_{last}$} ($\uparrow$)  & {$A_{avg}$} ($\uparrow$) & {$A_{last}$} ($\uparrow$)  & {$A_{avg}$} ($\uparrow$) & {$A_{last}$} ($\uparrow$)  & {$A_{avg}$} ($\uparrow$) & {$A_{last}$} ($\uparrow$)  & {$A_{avg}$} ($\uparrow$) \\
        \hline
    %    \multirow{6}*{\tabincell{c}{Sup-21K}} 
       % Join Linear probe \\
       % Join fine-tuning \\
       SLDA & $52.67_{\pm 0.59}$ &$59.38_{\pm 2.64}$ & $32.67_{\pm 1.25}$ &$37.22_{\pm 3.05}$ & $57.50_{\pm 0.21}$ &$62.65_{\pm 1.47}$ & $41.75_{\pm 0.53}$ &$50.78_{\pm 1.37}$\\ 
       FeCAM & $27.64_{\pm 0.55}$ &$29.28_{\pm 0.03}$ & $16.04_{\pm 1.38}$ &$20.06_{\pm 2.04}$ & $27.00_{\pm 0.22}$ &$29.34_{\pm 0.17}$ & $24.19_{\pm 1.11}$ &$30.43_{\pm 1.68}$\\
       SimpleCIL & $55.32_{\pm 0.89}$ &$66.25_{\pm 1.77}$ &  $25.48_{\pm 0.62}$ &$32.32_{\pm 1.89}$ & $58.65_{\pm 0.75}$ &$66.57_{\pm 1.03}$ & $34.48_{\pm 0.81}$ &$44.78_{\pm 1.63}$\\
       KLDA & $38.11_{\pm 1.24}$ &$44.84_{\pm 2.95}$ & $25.11_{\pm 1.93}$ &$30.29_{\pm 2.94}$ & $41.70_{\pm 0.47}$ &$47.12_{\pm 0.90}$ & $34.42_{\pm 1.07}$ &$44.64_{\pm 0.88}$\\
       AIR & $\underline{65.13}_{\pm 0.21}$ & $\underline{74.37}_{\pm 1.12}$ & $\underline{38.78}_{\pm 0.47}$ & $\underline{46.78}_{\pm 2.24}$ & $\underline{68.01}_{\pm 0.22}$ & $\underline{74.98}_{\pm 0.68}$ & $49.37_{\pm 0.45}$ & $59.57_{\pm 1.54}$ \\
       \midrule
       \midrule
       GACL & $30.78_{\pm 0.51}$ &$41.48_{\pm 2.35}$ & $24.28_{\pm 1.70}$ &$30.76_{\pm 2.64}$ & $30.01_{\pm 0.52}$ &$39.69_{\pm0.44}$ & $34.48_{\pm 1.27}$ &$45.77_{\pm 1.05}$\\
       \rowcolor{gray!20} ${\quad +Ours}$ & $52.40_{\pm 1.54}$ &$60.67_{\pm 2.69}$ & $34.86_{\pm 1.32}$ &$39.49_{\pm 2.75}$ & $58.04_{\pm 0.19}$ &$63.93_{\pm 1.28}$ & $48.33_{\pm 0.80}$&$56.72_{\pm 0.97}$\\
       \midrule
       RanPAC & $31.23_{\pm 0.57}$ & $41.84_{\pm 2.63}$ & $25.07_{\pm 1.59}$ & $32.24_{\pm 2.68}$ & $30.62_{\pm 0.59}$ & $40.95_{\pm 0.67}$ & $35.17_{\pm 1.29}$& $46.44_{\pm 1.23}$\\
       \rowcolor{gray!20} ${\quad +Ours}$ & $53.28_{\pm 1.55}$ & $61.32_{\pm 2.55}$ & $35.12_{\pm 1.32}$ & $40.94_{\pm 2.77}$ &  $58.82_{\pm 0.21}$ & $64.84_{\pm 1.17}$ & $49.02_{\pm 0.80}$& $57.04_{\pm .83}$\\
       \midrule
       AnaCP & $57.38_{\pm 1.14}$ &$63.30_{\pm 2.69}$ & ${37.84}_{\pm 1.46}$  &$41.95_{\pm 2.78}$ & $62.02_{\pm 0.34}$ &$66.71_{\pm 1.21}$ & $\underline{52.11}_{\pm 0.86}$&$\underline{60.27}_{\pm 1.52}$\\
        \rowcolor{gray!20} ${\quad +Ours}$  & $\textbf{68.57}_{\pm 0.27}$ &$\textbf{76.72}_{\pm 1.35}$ & $\textbf{41.56}_{\pm 0.65}$ &$\textbf{47.55}_{\pm 1.91}$ & $\textbf{69.68}_{\pm 0.24}$ &$\textbf{76.11}_{\pm 0.55}$ & $\textbf{57.71}_{\pm 0.68}$&$\textbf{66.76}_{\pm 1.09}$\\
       
       \midrule
       
       % \hline
       
	\end{tabular}
	} }
    \vspace{-3mm}
\end{table}

% \begin{table}[ht]
%     % \vspace{-0.2cm}
%     \centering
%     \caption{Ablation study.} 
%      % \vspace{-0.3cm}
% 	\smallskip
%       \renewcommand\arraystretch{1.}
%       % \setlength\tabcolsep{2.5pt}
%     \small{
% 	\resizebox{0.85\textwidth}{!}{ 
%         \centering
% 	\begin{tabular}{cccccccc}
% 	 \hline
%       \multirow{2}{*}{Meta-MLP} & \multicolumn{3}{c}{Feature augmentation} & \\
%         \hline

%        \hline
% 	\end{tabular}
% 	} 
% }
%         % \vspace{-2mm}
        
% 	\label{table:backbone}
% 	% \vspace{-0.1cm}
% \end{table}

\subsection{Ablation Study}

\subsubsection{Impact of Geometric Rectification Strategies.}
To validate the effectiveness of our method, we incrementally evaluate the components of GSR on the Split-CUB-200 dataset.
Table~\ref{tab:component_ablation} reveals a critical insight: applying standard \textit{Linear Mixup} in the high-dimensional feature space of DINO-v2 actually {degrades} performance ($80.77\% \rightarrow 77.87\%$). This confirms that linear interpolation within a hyperspherical embedding space can create "weak" features that drift off the manifold (norm shrinkage), confusing the classifier.
% In contrast, our GSR respects the geometry, boosting accuracy to $82.01\%$ and $47.45\%$. 
In contrast, our GSR respects the geometry, boosting accuracy to 82.46\% and 48.33\%.
% Finally, incorporating the {Adaptive Mixing} strategy ($\alpha_c$) further refines the spectral density of tail classes, achieving the best performances of ${82.46\%}$ and ${48.33\%}$. 

% \subsubsection{Component analysis}
% Table comparing: No Augment vs. Gaussian Sampling vs. Linear Mixup vs. Spherical Mixup (GSR).
% Mục đích: Chứng minh Spherical Mixup là cách duy nhất đúng về mặt hình học.
% Table \ref{tab:component_ablation}
\begin{table}[h]
\centering
\vspace{-3mm}
\caption{\textbf{Ablation Study on Component Effectiveness.} We incrementally add components to the baseline (Standard Ridge) to validate the contribution of each module in GSR. Evaluated on Split-CUB-200, using DINO-v2 as the pretrained backbone.}
\vspace{-1mm}
\label{tab:component_ablation}
\renewcommand\arraystretch{1.2}
    \small{
\resizebox{0.8\linewidth}{!}{
\begin{tabular}{l|cc|cc}
\toprule
\multirow{2}{*}{\textbf{Method Variant}} & \multicolumn{2}{c|}{\textbf{Components}} & \multicolumn{2}{c}{\textbf{Last Acc (\%)}} \\
 & \textbf{Mixup} & \textbf{Spherical} & \textbf{DINO-v2} & \textbf{MoCo-v3} \\
\midrule
Standard Ridge (GACL) & - & - & $80.77_{\pm 1.07}$ & $34.48_{\pm 1.27}$ \\
+ Linear Mixup & \checkmark & - & $77.87_{\pm 1.63}$   &  $46.07_{\pm 0.83}$ \\
% + Spherical Mixup & \checkmark & \checkmark & - & $\underline{82.01}_{\pm 0.47}$ & $\underline{47.45}_{\pm 0.67}$ \\
\rowcolor{gray!20} + \textbf{Spherical Mixup (GSR)} & \checkmark & \checkmark & {$\textbf{82.45}_{\pm 0.58}$} & $\textbf{48.32}_{\pm 0.80}$ \\
\bottomrule
\end{tabular}}
}
\vspace{-2mm}
\end{table} 

\subsubsection{Sensitivity to Imbalance Ratio ($\rho$)}
We investigate the robustness of GSR under varying degrees of data imbalance on Split-CIFAR-100. As shown in Table~\ref{tab:imbalance_sensitivity}, as the imbalance ratio $\rho$ increases from 10 (moderate) to 150 (extreme), the baseline GACL suffers a catastrophic drop of over $42\%$ ($81.23\% \rightarrow 38.60\%$). However, GSR demonstrates remarkable resilience. The performance gap between our method and the baseline widens significantly as the task becomes harder: from a $+3.99\%$ gain at $\rho=10$ to a massive ${+16.9\%}$ gain at $\rho=150$. This proves that GSR is not merely a constant regularizer, but an active \textit{spectral rectifier} that becomes increasingly critical as the tail lengthens.
% Table \ref{tab:imbalance_sensitivity}

\begin{table}[h]
\centering
\vspace{-4mm}
\caption{\textbf{Robustness to Data Imbalance. } Comparison of Last Accuracy (\%) under different imbalance ratios ($\rho$), on Split-CIFAR-100, using DINO-v2. Higher $\rho$ indicates a more severe imbalance. GACL+GSR demonstrates superior robustness compared to the base method GACL. }
\label{tab:imbalance_sensitivity}
\resizebox{0.6\linewidth}{!}{
\begin{tabular}{l|ccccc}
\toprule
\multirow{2}{*}{\textbf{Method}} & \multicolumn{4}{c}{\textbf{Imbalance Ratio ($\rho$)}} \\
\cmidrule(lr){2-6}
 & $\rho=10$ & $\rho=20$ & $\rho=50$ & $\rho=100$ & $\rho=150$ \\
\midrule
GACL & 81.23 & 74.19 & 59.99 & $48.78$ & 38.60 \\
\rowcolor{gray!20} {GACL +  GSR} & 85.22 & 82.58 & 75.16 & $65.51$ & 55.50\\
\midrule
\textit{Improvement ($\Delta$)} & \textit{+3.99} & \textit{+8.39} & \textit{+15.17} & \textit{+16.73} & \textit{{+16.9}} \\
\bottomrule
\end{tabular}
}
\end{table}
\vspace{-5mm}

% \subsubsection{Numerical Stability Analysis.}
% To empirically verify Theorem \ref{thm:stable_rank}, we monitor the \textit{Stable Rank} ($\text{sr}(G + \tau I)$) of the regularized Gram matrix throughout the incremental learning process.
% A stable rank close to 1.0 indicates severe spectral collapse, where the matrix is numerically indistinguishable from a rank-1 matrix (dominated by a single direction).
% Figure~\ref{fig:stable_rank} shows that the baseline GACL consistently hovers near the collapse point ($\text{sr} \approx 1.02$). In contrast, GSR effectively "thickens" the spectrum, maintaining a healthy stable rank ($\text{sr} \approx 2.80$) across all tasks. This confirms that our method successfully injects variance into the null-space of tail classes, ensuring a well-conditioned inversion.

\noindent\textbf{Spectral Analysis.}
We examine Theorem~\ref{thm:stable_rank} on the statistics that each task contributes to the
classifier. Let $\mathcal{S}_t$ be the set of features added at task $t$, each
with weight $w_i$ ($w_i = 1$ for real samples and $w_i = \beta$ for
rectification samples), and let $n_t = \sum_{i \in \mathcal{S}_t} w_i$. We
measure the stable rank of the mean-centred second moment
\begin{equation}
  \mathbf{C}_t \;=\; \sum_{i \in \mathcal{S}_t} w_i\,
      \big(\phi(\mathbf{z}_i) - \boldsymbol{\mu}_t\big)
      \big(\phi(\mathbf{z}_i) - \boldsymbol{\mu}_t\big)^{\top},
  \qquad
  \boldsymbol{\mu}_t \;=\; \frac{1}{n_t} \sum_{i \in \mathcal{S}_t} w_i\, \phi(\mathbf{z}_i).
  \label{eq:centred_gram}
\end{equation}
Adding the rectification set can only enlarge $\mathbf{C}_t$ in the
positive semi-definite order, i.e.\ $\widetilde{\mathbf{C}}_t = \mathbf{C}_t + \boldsymbol{\Delta}_t$
with $\boldsymbol{\Delta}_t \succeq 0$, which is the setting of Theorem~\ref{thm:stable_rank}.
We remove the mean because the random features are non-negative: the
uncentred Gram matrix is dominated by a single mean direction shared by all
classes, and $\mathrm{sr}(\mathbf{G} + \tau \mathbf{I})$ stays between $1.01$
and $1.06$ with or without GSR, which hides the rest of the spectrum.

Figure~\ref{fig:stable_rank} reports $\mathrm{sr}(\mathbf{C}_t)$ for GACL with
and without GSR, both computed on $\ell_2$-normalised features so that the
rectification set is the only difference. Without GSR, the statistics of a task
are dominated by its few head classes and the stable rank stays low. With GSR,
all classes of the task contribute comparably and the stable rank is
$2.2$--$2.4\times$ higher on average: $3.16 \rightarrow 7.10$ on ImageNet-R,
$3.20 \rightarrow 7.36$ on Tiny-ImageNet and $1.87 \rightarrow 4.55$ on
CIFAR-100. It is higher for every task and every seed.

% This effect concerns the per-task statistics. When the statistics are
% accumulated over tasks, the gap narrows, because the baseline gradually
% collects enough head classes to spread its spectrum, and the two methods are
% no longer distinguishable after the last task (e.g.\ $7.37$ vs.\ $7.40$ on
% ImageNet-R; Appendix~\ref{app:stable_rank}). The stable rank therefore
% characterises how GSR balances the contribution of each task, and should not
% be read as the sole explanation of the final accuracy gain, which we analyse
% through the head/tail breakdown below.
% Table \ref{tab:condition_number}

\begin{figure}[ht]
    % \vspace{-3mm}
    \centering
    \includegraphics[width=0.9\linewidth,trim=0cm 0cm 0cm 0cm, clip
    ]{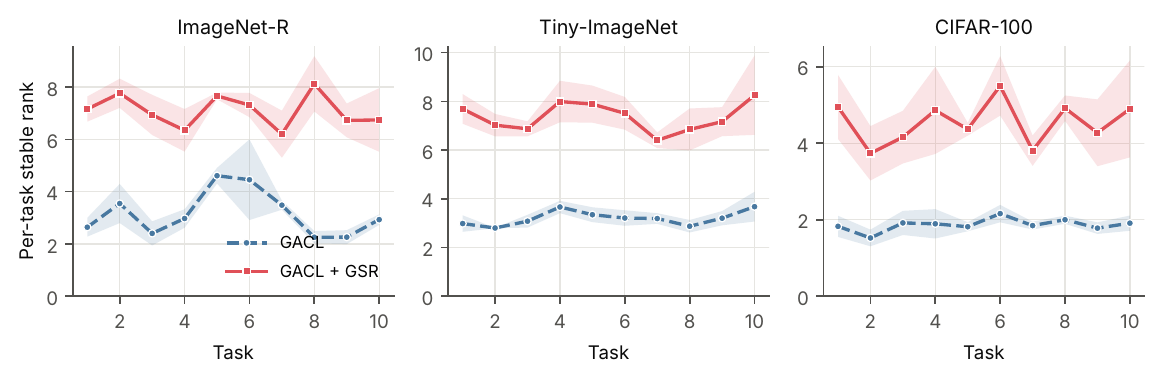}
    \vspace{-4mm}
    \caption{\textbf{Spectral analysis.} Stable rank of the mean-centred
  per-task feature statistics $\mathbf{C}_t$ (Eq.~\ref{eq:centred_gram}) for
  GACL with and without GSR, using DINO-v2. Lines are means over three seeds
  and shaded areas show one standard deviation. GSR yields a consistently
  higher stable rank for the statistics contributed by every task. 
    % \textcolor{plotred}{\textit{"img-r"}} stands for \textcolor{plotred}{\textit{Split-Imagenet-R}}, and \textcolor{plotblue}{\textit{"cub"}} stands for \textcolor{plotblue}{\textit{Split-CUB-200}}.
    }
    % Lower is better (more stable). The Baseline suffers from spectral collapse (exploding $\kappa$), while GSR maintains a well-conditioned spectrum.
    \label{fig:stable_rank}
    \vspace{-5mm}
\end{figure}

% \begin{table}[h]
% \centering
% \caption{\textbf{Numerical Stability Analysis.} We report the \textbf{\textit{Stable Rank}} of the $(G + \tau I)$ at different incremental stages. Lower is better (more stable). The Baseline suffers from spectral collapse (exploding $\kappa$), while GSR maintains a well-conditioned spectrum.}
% \label{tab:condition_number}
% \renewcommand\arraystretch{1.2}
%     \small{
% \resizebox{0.8\linewidth}{!}{
% \begin{tabular}{l|cccccccccc}
% \toprule
% \multirow{2}{*}{\textbf{Method}} & \multicolumn{10}{c}{{Task ID}} \\
%  & {T=1} & {T=2} & {T=3} & {T=4} & {T=5} & {T=6} & {T=7} & {T=8} & {T=9} & {T=10} \\
% \midrule
% GACL (img-r) & 1.05 & 1.03 & 1.02 & 1.02 & 1.02 & 1.02 & 1.02 & 1.02 & 1.01 & 1.01 \\
% \rowcolor{gray!20} {GACL + GSR (img-r)} & 2.90 & 2.80 & 2.80 & 2.79 & 2.78 & 2.78 & 2.78 & 2.77 & 2.77 & 2.77 \\
% GACL (cub) & 1.13 & 1.08 & 1.07 & 1.06 & 1.05 & 1.05 & 1.05 & 1.05 & 1.05 & 1.04 \\
% \rowcolor{gray!20} GACL + SNR (cub)& 3.36 & 3.07 & 3.00 & 2.94 & 2.90 & 2.88 & 2.86 & 2.84 & 2.83 & 2.83\\
% \bottomrule
% \end{tabular}}
% }
% \end{table}
\subsubsection{Performance Breakdown by Class Frequency.}
% Table comparing: Fixed mixing rate (\mu) vs. Adaptive mixing rate (\alpha_c).
% Mục đích: Chứng minh Tail class cần được đối xử khác Head class.
Analysis of Class Frequency Impact. To investigate the source of the performance gains, Table \ref{tab:adaptive_efficacy} breaks down the accuracy for the most frequent (Head) and least frequent (Tail) classes.

First, GSR demonstrates a remarkable ability to resurrect tail classes. In standard GACL, tail performance is extremely low (e.g., 12.50\% on Split-CIFAR-100 and 13.33\% on Split-Tiny-ImageNet), indicating the model mostly ignores these minority classes. However, integrating GSR yields massive improvements, boosting tail accuracy by 3$\times$ to 4$\times$ (reaching 38.00\% and 52.70\%, respectively). This confirms that our GSR effectively prevents feature collapse in sparse regions.

Second, regarding the head classes, we observe a negligible trade-off. While there is a slight dip in accuracy for CIFAR-100 and Tiny-ImageNet (approx. 1\%), this is a typical phenomenon in long-tailed learning due to the alleviation of majority-class bias. Crucially, the massive gains in the tail far outweigh this minor drop, resulting in a significantly higher overall accuracy.

Interestingly, on ImageNet-R, GSR improves both head (+5.62\%) and tail (+26.29\%) performance. 
This indicates that GSR does not merely shift the decision boundary bias. Instead, it rectifies the feature distribution to be more representative, enabling the analytic solver to generalize better across both head and tail classes.
% Since ImageNet-R contains out-of-distribution samples, 
% This suggests that GSR does not merely re-balance the decision boundary but can even learns more robust and generalizable feature representations via the rectified geometric manifold.
% Table \ref{tab:adaptive_efficacy}
% \begin{table}[h]
% \centering
% \caption{\textbf{Breakdown of accuracy by class frequency groups (Head, Tail).}  }
% \label{tab:adaptive_efficacy}
% \resizebox{0.8\linewidth}{!}{
% \begin{tabular}{l|cc|c}
% \toprule
% \multirow{2}{*}{\textbf{Method}} & \multicolumn{2}{c|}{\textbf{Group Accuracy (\%)}} & \multirow{2}{*}{\textbf{Overall (\%)}} \\
% \cmidrule(lr){2-3}
%  & \textbf{Head (5 classes)} &  \textbf{Tail (5 classes)} & \\
% \midrule
% GACL (cifar100) & 95.38 & 12.50 & 48.78 \\
% \rowcolor{gray!20} GACL + GSR (cifar100) & 95.25  & 38.00 & 65.81\\
% \midrule 
% GACL (imagenet-r) & 83.38 & 21.09 & 47.84 \\
% \rowcolor{gray!20} GACL + GSR (imagenet-r) & 89.00  & 47.38 & 61.58\\
% \midrule 
% GACL (tiny-imagenet) & 95.00 & 13.33 & 42.98 \\
% \rowcolor{gray!20} GACL + GSR (tiny-imagenet) & 94.00  & 52.70 & 68.88 \\
% \bottomrule
% \end{tabular}
% }
% \end{table}

\begin{table}[ht]
\centering
\vspace{-3mm}
\caption{\textbf{Breakdown of accuracy by class frequency groups (Head, Tail).} DINO-v2 is used as the pretrained backbone. }
\vspace{-1mm}
\label{tab:adaptive_efficacy}
\resizebox{0.82\linewidth}{!}{
\begin{tabular}{llccc}
\toprule
\multirow{2}{*}{\textbf{Dataset}} & \multirow{2}{*}{\textbf{Method}} & \multicolumn{2}{c}{\textbf{Group Accuracy (\%)}} & \multirow{2}{*}{\textbf{Overall (\%)}} \\ \cmidrule(lr){3-4}
 &  & \textbf{Head (5 classes)} & \textbf{Tail (5 classes)} &  \\ \midrule
\multirow{2}{*}{CIFAR100} & GACL & 95.38 & 12.50 & 48.78 \\
 & \cellcolor{gray!15}GACL + GSR & \cellcolor{gray!15}95.25 & \cellcolor{gray!15}38.00 & \cellcolor{gray!15}65.51 \\ \midrule
\multirow{2}{*}{Imagenet-R} & GACL & 83.38 & 21.09 & 47.84 \\
 & \cellcolor{gray!15}GACL + GSR & \cellcolor{gray!15}89.00 & \cellcolor{gray!15}47.38 & \cellcolor{gray!15}61.58 \\ \midrule
\multirow{2}{*}{Tiny-Imagenet} & GACL & 95.00 & 13.33 & 42.98 \\
 & \cellcolor{gray!15}GACL + GSR & \cellcolor{gray!15}94.00 & \cellcolor{gray!15}52.70 & \cellcolor{gray!15}68.88 \\ \bottomrule
\end{tabular}}
% \vspace{-8mm}
\end{table}

\noindent\textbf{GSR vs.\ common alternatives.}
Table~\ref{tab:efficiency} compares GSR with other ways of generating the
rectification features: Gaussian sampling with a shrunk or a transferred
covariance, and inter-class mixup, which mixes a tail sample with a sample of
another class and keeps the tail label. All strategies use the same per-class
budget.

\textit{Efficiency vs. Accuracy Trade-off.} The final accuracy depends on the base learner. On GACL, the strategies that borrow
variability from other classes are stronger: covariance transfer and
inter-class mixup improve over GSR by $9.4$ and $5.4$ points on ImageNet-R and
by $6.6$ and $3.4$ points on Tiny-ImageNet. They raise tail accuracy further
than GSR (e.g.\ $64.1\%$ and $58.0\%$ vs.\ $49.7\%$ on ImageNet-R), at a
moderate cost on head classes ($86.3\%$ and $86.7\%$ vs.\ $88.1\%$). A single
ridge classifier has no other mechanism to compensate for the missing
variability of tail classes, so this trade is favourable.

On AnaCP the picture reverses. GSR obtains the best accuracy on ImageNet-R and
Tiny-ImageNet and is on par with the best alternative on CUB-200 ($85.89\%$
vs.\ $85.92\%$). Here the alternatives gain little on tail classes but lose
substantially on head classes: on ImageNet-R, head accuracy drops from
$83.0\%$ with GSR to $77.9\%$ with inter-class mixup and $70.3\%$ with
covariance transfer. Intra-class mixing leaves the class means in place, which
AnaCP relies on for its contrastive projection and pseudo-replay, whereas
features borrowed from other classes shift them.

In terms of cost, the Gaussian strategies require an $O(D^3)$ factorisation of
a covariance matrix per task, while both mixup variants are $O(D)$. Overall,
GSR is the preferable choice in combination with AnaCP, which is also the
configuration that achieves the highest accuracy in Table~1; for a plain
ridge classifier such as GACL, borrowing variability from other classes is
more effective.

\textit{Synergy with AnaCP.} We want to highlight \textit{why AnaCP integrates naturally with our method.} A key ingredient of AnaCP \cite{DBLP:journals/corr/abs-2511-13880} is the Contrastive Projection Layer, which encourages the learned features to form well-separated clusters—a geometric structure closely related to the Neural Collapse phenomenon in contrastive learning \cite{nguyen2024neural}. However, under head–tail (long-tailed) data imbalance, contrastive learning is prone to \textit{spectral (dimensional) collapse} \cite{jing2021understanding}: the representation remains non-degenerate (features are not identical), yet most variation concentrates in only a few directions, making the embedding effectively low-rank. Our method directly addresses this issue. 
% In particular, Theorem \ref{thm:stable_rank} shows that we increase the stable rank of the learned representation, mitigating dimensional collapse and thereby strengthening the feature geometry that AnaCP relies on—ultimately improving AnaCP’s performance.
Intra-class mixing is compatible with this structure because it leaves the
class means in place (see Table~\ref{tab:efficiency} and Appendix~\ref{explain}).
\begin{table}[h]
\centering
\vspace{-3mm}
\caption{\textbf{Comparison of computational complexity and accuracy} of different feature augmentation methods, using DINO-v2 as the pretrained backbone. $D$ is the feature dimension. See Appendix \ref{sec:appendix_analysis} for details on the computational complexity.}
% Covariance-based Gaussian augmentation requires $O(D^3)$ operations, which becomes expensive for high-dimensional DinoV2 features, whereas Mixup and our GSR scale linearly, $O(D)$. Despite this efficiency gain, GSR remains competitive on GACL and achieves the best overall results when combined with AnaCP (notably on Imagenet-R and Tiny-Imagenet), demonstrating a strong accuracy--efficiency trade-off.
\label{tab:efficiency}
\smallskip
      \renewcommand\arraystretch{1.2}
    \small{
\resizebox{0.9\linewidth}{!}{
\begin{tabular}{l  c  c c c}
\toprule
\multirow{2}{*}{\textbf{Method}} 
& {\textbf{Computational}} 
& \multicolumn{3}{c}{\textbf{Last Accuracy (\%) $\uparrow$}} \\
\cmidrule(lr){3-5}
& \textbf{Complexity} & CUB-200 & Imagenet-R & Tiny-Imagenet \\
\midrule
GACL +  Gauss (cov shrinkage) & $O(D^3)$ &  82.50 & 62.85 & 69.18 \\
GACL +  Gauss (cov transfer) & $O(D^3)$ &  \textbf{83.04} & \textbf{70.85} & \textbf{75.29} \\
GACL + Inter-class mixup & $O(D)$ & 82.85 & 66.85 & 72.11 \\ 
\rowcolor{gray!20} {GACL + GSR (Ours)} & $O(D)$ & {82.45} & {61.46} & 68.68\\
\midrule
AnaCP +  Gauss (cov shrinkage) & $O(D^3)$ &  \textbf{85.92} & 71.35 & \textbf{79.29}\\
AnaCP +  Gauss (cov transfer) & $O(D^3)$ & \underline{85.84} & 70.99 & 76.55\\
AnaCP + Inter-class mixup & $O(D)$ & 85.75 & \underline{71.43} & {78.91} \\ 
\rowcolor{gray!20} {AnaCP + GSR (Ours)} & $O(D)$ & {85.89} & \textbf{{72.16}} & \underline{79.58} \\
\bottomrule
\end{tabular}}
}
% \vspace{-8mm}
\end{table}

\section{Conclusion}
% In this work, we identified Spectral Collapse as the fundamental cause of failure for the state-of-the-art Analytic Continual Learning (ACL) methods in long-tailed scenarios. We demonstrated that the standard isotropic regularization used in Ridge Regression is insufficient to handle the extreme eigenvalue skewness caused by class imbalance. To address this, we proposed Geometry-Spectral Rectification (GSR), a theoretically grounded framework that acts as an anisotropic spectral filter. By leveraging Spherical Geodesic Mixing and Adaptive Spectral Densification, GSR effectively "heals" the collapsed subspaces of tail classes without hallucinating arbitrary noise or inducing semantic bias from head classes.
% Our extensive experiments demonstrate our effectiveness in the practical but challenging setting of Class Incremental Learning under long-tailed distributions. 
% Future work may explore extending this spectral rectification perspective to update the backbone itself in a resource-constrained manner.

In this work, we identified Spectral Collapse as a critical bottleneck for state-of-the-art Analytic Continual Learning (ACL) methods in long-tailed scenarios. We demonstrated that the standard isotropic regularization in Ridge Regression is insufficient to handle the extreme eigenvalue skewness caused by class imbalance. To address this, we proposed Geometry-Spectral Rectification (GSR), a theoretically grounded framework acting as an anisotropic spectral filter. By leveraging Spherical Geodesic Mixing and Class-adaptive Rectification Budget, GSR effectively "heals" the collapsed subspaces of tail classes without hallucinating arbitrary noise or inducing semantic bias from head classes. Extensive experiments validate the effectiveness of GSR in the practical yet challenging setting of Long-Tailed Class Incremental Learning.

% \color{red}{
% \textbf{TODO:}

% - Eigenspectrum of Gram matrix? (with vs w/o ours)

% - Condition number over time? (baselines vs ours) 

% - tSNE/PCA (Linear mixup vs Spherical mixup vs Gaussian sampling) 

% - Finish the main tables
% }

% - Modify the symbols for eigenvalues (lambda), ridge coeff (xi), mixing coef (mu)
% \color{black}

% \clearpage
% \newpage
% \section*{Acknowledgements}
% Please insert your acknowledgments here.
% \bibliography{main}

\clearpage
\newpage
\bibliographystyle{splncs04}
\bibliography{main}

\newpage
\appendix
% \maketitlesupplementary
\setcounter{page}{1}
{\LARGE \textbf{Appendices for }} 
\begin{center}

{\large \textbf{\textit{"Spectral-Aware Analytic Class-Incremental Learning for Long-Tailed Distributions"}}}
\end{center}

% \noindent \textbf{Overview of the Appendix.} 
This appendix provides supplementary materials, theoretical proofs, and detailed experimental results to support the main manuscript. The document is organized as follows: Section \ref{sec:appendix_analysis} presents an in-depth comparative analysis of various feature augmentation strategies, including discussions on their computational complexity and geometric properties. Section \ref{B} provides a theoretical justification for our chosen data augmentation strategy and details the complete mathematical proof for Theorem 1. Section \ref{exp_detail} elaborates on the experimental setups, including dataset protocols, baseline configurations, and implementation details. Finally, Section \ref{additional_exp} reports comprehensive additional experimental results, including the full performance comparison of baselines in standard versus long-tailed settings.

\section{Comparative Analysis of Feature Augmentation Strategies}
\label{sec:appendix_analysis}

This section provides an in-depth analysis of the results presented in Table \ref{tab:efficiency}, comparing the proposed GSR with several baseline feature augmentation methods (Gaussian sampling with covariance shrinkage/covariance transfer), and linear mixup. Specifically, we evaluate the computational complexity (Section \ref{time}) alongside the relative disadvantages of each baseline (Section \ref{advantage}). 
% Furthermore, we explain why GSR is outperformed by Gaussian sampling with covariance transfer when combined with GACL, yet achieves superior performance when combined with AnaCP.
Furthermore, we discuss why GSR is outperformed by covariance transfer and
inter-class mixup when combined with GACL, yet achieves the best or
on-par accuracy when combined with AnaCP.
\subsection{Computational Complexity Analysis}
\label{time}

% Efficiency is critical for on-the-fly augmentation during training.

\textbf{Gaussian sampling methods (Covariance Shrinkage/ Transfer):}
Both methods require sampling from a multivariate Gaussian. This involves:
\begin{enumerate}
    \item Constructing the covariance matrix $\hat{{\Sigma}}$ $\longrightarrow \mathcal{C}_1 \approx O(D^2)$.
    \item Performing Cholesky decomposition or Eigen-decomposition to find ${L}$ such that $\hat{{\Sigma}} = {L}{L}^T$. This decomposition scales cubically: $\mathcal{C}_{2} \approx O(D^3)$
    \item Sampling: $\tilde{z} = \mu + L\epsilon, \quad \epsilon \sim N(0, I) \longrightarrow \mathcal{C}_{3} \approx O(D^2)$
\end{enumerate}

Therefore, the overall time complexity of Gaussian sampling methods is: $$\mathcal{C}_{Gaussian} \approx O(D^3)$$

With $D \geq 768$\footnote{$D=768$ is the original dimension of the latent space of ViT/B-16, which we used in our experiments. In GACL and AnaCP, we use $ D=5000$ to implement the non-linear projection operation of these two methods.}, this is computationally expensive to perform repeatedly for every tail class in every batch.

\noindent\textbf{GSR (Ours):}
Our Spherical Mixup involves only vector linear combinations and normalization, thus the corresponding complexity is:
\begin{equation}
    \mathcal{C}_{GSR} \approx O(D)
\end{equation}
This linear complexity makes GSR orders of magnitude faster and negligible in terms of training overhead. Similarly, $O(D)$ is also the time complexity of the linear mixup method.

% provide a geometric intuition as to why GSR yields sub-optimal results compared to covariance transfer in the original feature space (GACL), but demonstrates superior efficacy in the whitened, isotropic space (AnaCP).

% This section provides a more detailed analysis of the results in Table \ref{tab:efficiency}, comparing GSR to other simpler alternatives. We include an analysis of computational complexity, discuss the advantages and disadvantages of each method, and explain why GSR is outperformed by Gaussian sampling with covariance transfer when combined with GACL, yet achieves superior performance when combined with AnaCP.

% In this section, we provide a comparative analysis of our proposed GSR augmentation (Spherical Intra-class Mixup) and two dominant parametric baselines in long-tailed learning: \textit{Covariance Shrinkage} and \textit{Covariance Transfer}. We analyze them based on geometric fidelity, statistical assumptions, and computational complexity.

\subsection{The advantage of GSR against the other feature augmentation methods (Gaussian sampling, linear mixup)}
\label{advantage}

\begin{table}[ht]
\centering
\caption{Comparison of GSR against the other Gaussian sampling techniques}
\vspace{-3mm}
\label{tab:comparison_detailed}
\resizebox{\columnwidth}{!}{%
\begin{tabular}{l|ccc}
\hline
\textbf{Method} & \textbf{Risk / Drawback} & \textbf{Assumption} & \textbf{Complexity} \\ \hline
Cov. Shrinkage & Manifold Intrusion (Blind Noise) & Isotropic Noise & $O(D^3)$ \\
Cov. Transfer & Semantic Bias (Head $\to$ Tail) & Class Homoscedasticity & $O(D^3)$ \\
\textbf{GSR (Ours)} & \textbf{Safe Interpolation, but limited by data samples} & \textbf{None (Data-driven)} & $\mathbf{O(D)}$ \\ \hline
\end{tabular}%
}
\end{table}

% Standard approaches assume features follow a Gaussian distribution $\mathcal{N}(\boldsymbol{\mu}, \boldsymbol{\Sigma})$. However, estimating $\boldsymbol{\Sigma}$ with few samples ($N \ll d$) is ill-posed.

% Standard approaches assume features follow a Gaussian distribution $\mathcal{N}(\boldsymbol{\mu}, \boldsymbol{\Sigma})$. However, estimating $\boldsymbol{\Sigma}$ with few samples ($N \ll d$) is ill-posed.

\subsubsection{vs. Covariance Shrinkage (The "Blind Noise" Problem)}
Covariance Shrinkage attempts to fix the rank-deficiency of the tail class covariance $\boldsymbol{\Sigma}_{tail}$ by adding a regularization term:
\begin{equation}
    \hat{\boldsymbol{\Sigma}}_{CS} = (1-\alpha)\boldsymbol{\Sigma}_{tail} + \alpha \mathbf{I}
\end{equation}
where $\mathbf{I}$ is the identity matrix, $\alpha$ is a scalar coefficient.
\begin{itemize}
    \item \textit{Potential Disadvantage -  Geometric Flaw:} The identity matrix $\mathbf{I}$ represents isotropic noise (a hypersphere of uncertainty). In high-dimensional spaces, the true data manifold is sparse. Adding $\alpha \mathbf{I}$ effectively adds noise in all orthogonal directions perpendicular to the data manifold. This "blind" expansion increases the risk of generated samples falling into the decision regions of other classes (manifold intrusion), reducing class separability.
    \item \textit{GSR Advantage:} Our method does not add external noise. By interpolating strictly between existing samples ($\mathbf{x}_{new} = \text{Norm}(\gamma \mathbf{x}_i + (1-\gamma)\mathbf{x}_j)$), the generated data remains confined to the geodesic path connecting real samples. This ensures the augmented data stays within the valid semantic subspace of the class.
\end{itemize}
\vspace{-2mm}
\subsubsection{vs. Covariance Transfer (The "Bias" Problem)}
Covariance Transfer assumes that tail classes share the same geometric distribution shape as head classes. It approximates the tail covariance using head class statistics:
\begin{equation}
    \hat{\boldsymbol{\Sigma}}_{CT} = \boldsymbol{\Sigma}_{tail} + \beta \boldsymbol{\Sigma}_{head}
\end{equation}
\begin{itemize}
    \item \textit{Potential Disadvantage - Statistical Flaw:} This relies on a strong assumption of homoscedasticity (similar variance) across classes. In reality, a "Head" class (e.g., \textit{Dog}, with high visual variance) may have a completely different feature spread than a "Tail" class (e.g., \textit{Platypus}, with specific features). Forcing the variance of a head class onto a tail class can introduce a significant bias, potentially distorting the tail class representation to look like the head class (as will be discussed in the next sub-section).
    \item \textit{GSR Advantage:} Our approach is non-parametric and data-driven. It does not rely on the statistics of head classes, thereby avoiding the introduction of bias from majority classes. It respects the unique, albeit sparse, structure of the tail class itself.
\end{itemize}

% \subsection{Computational Complexity Analysis}

% Efficiency is critical for on-the-fly augmentation during training.

% \textbf{Gaussian-based Methods (Shrinkage \& Transfer):}
% Both methods require sampling from a multivariate Gaussian. This involves:
% 1. Constructing the covariance matrix $\hat{\boldsymbol{\Sigma}}$ ($O(D^2)$).
% 2. Performing Cholesky decomposition or Eigendecomposition to find $\mathbf{L}$ such that $\hat{\boldsymbol{\Sigma}} = \mathbf{L}\mathbf{L}^T$.
% This decomposition scales cubically:
% \begin{equation}
%     \mathcal{C}_{Gaussian} \approx O(d^3)
% \end{equation}
% With $d=512$, this is computationally expensive to perform repeatedly for every tail class in every batch.

% \textbf{GSR (Ours):}
% Our Spherical Mixup involves only vector linear combinations and normalization:
% \begin{equation}
%     \mathcal{C}_{GSR} \approx O(d)
% \end{equation}
% This linear complexity makes GSR orders of magnitude faster and negligible in terms of training overhead.

\noindent Table \ref{tab:comparison_detailed} summarizes the qualitative differences between GSR and the Gaussian sampling techniques.

\subsubsection{vs. Standard Linear Mixup (The "Manifold Sagging" Problem)}
Standard Mixup performs interpolation in Euclidean space: $\mathbf{x}_{lin} = \gamma \mathbf{x}_i + (1-\gamma)\mathbf{x}_j$.
\begin{itemize}
    \item \textit{Potential Disadvantage - Geometric Flaw:} In modern representation learning, features are normalized to lie on a hypersphere $\mathbb{S}^{d-1}$ ($||\mathbf{x}||=1$). The linear chord connecting two points on a sphere lies strictly inside the sphere. Consequently, $||\mathbf{x}_{lin}|| < 1$.
    Therefore, the magnitude of a feature vector often correlates with the model's confidence. By generating samples with smaller norms ("sagging" off the manifold), Linear Mixup produces "weak" features that may be ignored by the cosine classifier or treated as low-quality noise.
    \item \textit{GSR Advantage:} By projecting the interpolated point back onto the sphere ($\ell_2$-normalization), GSR maintains the unit norm constraint, ensuring the generated samples are indistinguishable from real samples in terms of magnitude and lie on the correct geodesic manifold.
\end{itemize}

\subsection{Further explanation of the results in Table \ref{tab:efficiency}}
\label{explain}

The results in Table \ref{tab:efficiency} show that Gaussian sampling with cov transfer yields significantly better results than GSR when combined with GACL; however, it falls behind GSR when combined with AnaCP on ImageNet-R and Tiny-ImageNet. Inter-class mixup behaves in the same way to a smaller degree: it improves over GSR on GACL and falls behind it on AnaCP. Both strategies borrow variability from other classes, which raises tail accuracy but lowers head accuracy; the balance is favourable for GACL and unfavourable for AnaCP. This section will further explain the reason.
% \textcolor{red}{\textbf{[NEED TO CHECK and REVISE, chán quá, dcm]}}

\subsubsection{a. The Feature Space of GACL and the Suitability of Covariance Transfer (CT)}

GACL primarily utilizes features directly from the Pre-trained Model (PTM) or through a simple Random Projection (RP) layer, followed by Ridge Regression for classification.

\begin{itemize}
    \item \textit{Space Characteristics:} In the PTM space, variance directions typically represent shared, class-agnostic semantic factors (e.g., background, camera angle, lighting).
    \item \textit{Mathematics of CT:} Covariance Transfer assumes that the covariance matrix of the majority (Head) class, \(\boldsymbol{\Sigma}_{head}\), contains these rich variance directions. When transferred to the minority (Tail) class:
    \[
    % \begin{equation}
    \hat{\boldsymbol{\Sigma}}_{CT} = \boldsymbol{\Sigma}_{tail} + \beta \boldsymbol{\Sigma}_{head}
% \end{equation}
    \]
    $\Rightarrow$ We can effectively ``borrow'' these shared semantic variations to enrich the Tail class data. Because these variance directions are \textit{shared} across the entire PTM space, adding \(\boldsymbol{\Sigma}_{head}\) does not alter the core semantic identity of the Tail class, thereby expanding the decision boundary in a reasonable manner.
\end{itemize}

\subsubsection{b. The Feature Space of AnaCP and the Failure of Covariance Transfer.}

In this part, we analyze why applying Covariance Transfer (CT) in the original feature space results in performance degradation in AnaCP. 
% when the Contrastive Prototype (CP) layer is learned via a closed-form Ridge solution. 
The failure is caused by a fundamental linear contradiction introduced into the projection matrix:

In the first step, CT augments the Tail class in the original space by injecting variance from the Head class:
\begin{equation}
    \mathbf{z}_{tail}^{aug} = \mathbf{z}_{tail} + \delta, \quad \text{where } \delta \sim \mathcal{N}(0, \beta \Sigma_{head})
\end{equation}
% These augmented samples are basically combined with the original data to form the training matrix \(\tilde{\mathbf{Z}}\). The CP layer aims to find a linear projection \(\mathbf{W}_{CP}\) that maps \(\tilde{\mathbf{Z}}\) to the repulsed prototypes \(\mathbf{Y}_{repulsed}\) using the closed-form Ridge Regression solution:
% \begin{equation}
%     \mathbf{W}_{CP} = (\tilde{\mathbf{Z}}^T \tilde{\mathbf{Z}} + \lambda \mathbf{I})^{-1} \tilde{\mathbf{Z}}^T \mathbf{Y}_{repulsed}
% \end{equation}

% The mathematical flaw arises from the conflicting targets assigned to the same variance directions. For the original Head class samples, \(\mathbf{W}_{CP}\) is optimized to map features varying along the principal components of \(\Sigma_{head}\) to the Head prototype \(\mathbf{y}_{head}\). However, due to the CT augmentation, the matrix \(\mathbf{X}\) now also contains Tail samples that vary along the same directions (\(\Sigma_{head}\)), but their target in \(\mathbf{Y}_{repulsed}\) is the Tail prototype \(\mathbf{y}_{tail}\).

These augmented samples are combined with the original data to form the training matrix \(\tilde{\mathbf{Z}}\). Typically, these features pass through a non-linear mapping \(\phi(\cdot)\) to form the activated feature matrix \(\mathbf{H} = \phi(\tilde{\mathbf{Z}}.\mathbf{W}_{RP})\), where $\mathbf{W}_{RP}$ is a random matrix. The CP layer then aims to find a projection \(\mathbf{W}_{CP}\) that maps \(\mathbf{H}\) to the repulsed prototypes \(\mathbf{Y}_{repulsed}\) using the closed-form Ridge Regression solution:
\begin{equation}
    \mathbf{W}_{CP} = (\mathbf{H}^T \mathbf{H} + \lambda \mathbf{I})^{-1} \mathbf{H}^T \mathbf{Y}_{repulsed}
\end{equation}

The mathematical flaw arises from the \textit{feature entanglement} introduced by the injected variance. By injecting \(\Sigma_{head}\) into the Tail class, the augmented Tail samples inherently share structural properties with the Head class. Even after passing through the non-linear activation \(\phi(\cdot)\), this shared variance translates into highly overlapping activation patterns in the \(\mathbf{H}\) space. 

Consequently, the Ridge solver faces an optimization conflict: it is tasked with mapping these entangled, structurally similar activated features to distinctly repulsed prototypes (\(\mathbf{y}_{head}\) and \(\mathbf{y}_{tail}\)). To minimize the overall mean squared error, the closed-form solution \(\mathbf{W}_{CP}\) is forced to compromise, effectively distorting the weights associated with these shared activation components.

As a result, the learned projection \(\mathbf{W}_{CP}\) becomes corrupted. In the subsequent stage, when features are sampled and forwarded through the CP layer, the projected distributions of the Head and Tail classes collapse towards each other along these distorted components. This severe feature overlap makes it impossible for the final classification layer to establish clear decision boundaries, leading to a drop in performance.

\subsubsection{c. Geometry-Spectral Rectification (GSR) Succeeds on AnaCP}

In contrast to the aforementioned limitations of Covariance Transfer, our proposed Geometry-Spectral Rectification (GSR) method does not ``borrow'' variance from other classes. Instead, it operates based on the \textit{intrinsic structure} of the target class itself. As established in our previous analysis, by preserving the class-specific information of each tail class, GSR can better circumvent the fundamental drawbacks of Gaussian-based Covariance Transfer, which inadvertently corrupts the original feature space and destroys the intrinsic diversity of the data. 

This success is driven by two main factors:
\begin{itemize}
    \item \textit{Compatibility with AnaCP's Geometry:} Because the CP layer in AnaCP utilizes closed-form Ridge Regression (where the \(L_2\) penalty encourages data clustering) combined with normalized prototypes, the projected data for each class naturally forms tight spherical or ellipsoidal distributions around its respective prototype.
    \item \textit{Mathematics of GSR:} Generating samples via Spherical Interpolation strictly \textit{respects the subspace structure} established by AnaCP. It increases the sample density of the Tail class without leaking variance into the orthogonal directions occupied by other classes, thereby completely avoiding the linear contradictions and feature overlaps caused by Covariance Transfer.
\end{itemize}

% \subsubsection{Theoretical Summary}

% \begin{enumerate}
%     \item \textbf{GACL:} The feature space preserves shared variations \(\rightarrow\) CT performs well because classes can share these variance directions.
%     \item \textbf{AnaCP:} Whitening and Negative Repulsion eliminate shared variations and force classes into specific orthogonal directions \(\rightarrow\) CT fails because it imposes the highly specific variance of one class onto another, destroying the underlying geometric structure.
%     \item \textbf{AnaCP + GSR:} GSR respects the internal spherical/orthogonal cluster structure generated by AnaCP, allowing for safe augmentation of Tail class data without encroaching on the feature spaces of other classes.
% \end{enumerate}

\section{Analysis of Data Augmentation Strategies for Robust Covariance Estimation}
\label{B}

In this work, we employ \textit{Spherical Intra-class Mixup} to augment tail class samples prior to covariance estimation. This section provides a theoretical justification for this choice over the \textit{Conventional} \textit{Gaussian Sampling} approach.

\subsection{Limitations of Gaussian Sampling in Tail classes}
Standard augmentation often involves estimating the empirical mean $\boldsymbol{\mu}_c$ and covariance $\boldsymbol{\Sigma}_c$ of a class $c$, and subsequently sampling $\tilde{\mathbf{z}} \sim \mathcal{N}(\boldsymbol{\mu}_c, \boldsymbol{\Sigma}_c)$. However, this approach is suboptimal for tail classes due to two primary factors:

\begin{itemize}
    \item \textit{Ill-Posed Estimation:} In the training regime where the number of samples $N_c$ is significantly smaller than the feature dimension $D$, the empirical covariance $\boldsymbol{\Sigma}_c$ is rank-deficient and highly sensitive to outliers. Sampling from such an ill-posed distribution risks generating ``hallucinated'' features that deviate significantly from the true class manifold, potentially causing class overlap.
    
    \item \textit{Geometric Mismatch:} Modern representation learning typically constrains features to a unit hypersphere ($\mathcal{S}^{d-1}$). The Gaussian distribution is defined in Euclidean space. Sampling from a Gaussian and projecting back to the hypersphere ($\mathbf{z} \leftarrow \mathbf{z} / \|\mathbf{z}\|_2$) distorts the underlying probability density, often leading to an artificial concentration of mass around the class mean (mode collapse) rather than preserving the intra-class variance structure.
\end{itemize}

\subsection{Geometric Properties of Spherical Mixup}
\label{spherical_mix}
Our approach generates a synthetic sample $\tilde{\mathbf{z}}$ from two support samples $\mathbf{z}_1, \mathbf{z}_2$ via:
\begin{equation}
    \tilde{\mathbf{z}} = \frac{\gamma \mathbf{z}_1 + (1-\gamma)\mathbf{z}_2}{\| \gamma \mathbf{z}_1 + (1-\gamma)\mathbf{z}_2 \|_2}, \quad \text{where } \gamma \sim \text{Beta}(2, 2)
\end{equation}
Geometrically, this operation corresponds to \textit{geodesic interpolation}. Instead of traversing the Euclidean chord (which would violate the unit norm constraint), the synthetic sample moves along the shortest path on the manifold surface between $\mathbf{z}_1$ and $\mathbf{z}_2$.

\subsection{Conservative Densification vs. Extrapolation}
A critical distinction between our approach and standard Mixup lies in the interpolation geometry. Standard linear Mixup generates samples via convex combinations: $\tilde{z} = \alpha z_1 + (1-\alpha)z_2$. In high-dimensional spaces where features are normalized (lying on $\mathcal{S}^{d-1}$), this Euclidean interpolation traverses the "chord" inside the hypersphere. This results in the Norm Shrinkage Phenomenon: synthetic samples have strictly smaller norms ($||\tilde{z}|| < 1$), causing them to drift away from the true data manifold and artificially reducing the variance along principal components.

In contrast, our Spherical Mixup performs geodesic interpolation. By projecting the mixed sample back onto the unit hypersphere, we introduce a non-linear normalization step: $\tilde{z} \leftarrow \tilde{z} / ||\tilde{z}||$. This ensures that:

Manifold Adherence: Synthetic samples remain on the same geometric surface as real data, preserving the distributional statistics of the class.
Variance Preservation: Unlike linear mixup which dampens signal energy, Spherical Mixup maintains unit variance, effectively "inflating" the covariance matrix along the intra-class directions. 
% This acts as a structured spectral regularizer that improves the stable rank $s(G)$ without introducing isotropic noise (as in Ridge Regression). 
This acts as a structured regularizer that enlarges the contribution of tail classes without introducing isotropic noise (as in Ridge Regression).
Furthermore, empirical studies in deep feature spaces suggest that intra-class interpolation preserves semantic identity better than extrapolation. By restricting our mixup strictly within the geodesic path between two real samples of the same class, we minimize the risk of crossing decision boundaries (hallucination), ensuring that the injected variance contributes constructively to the class subspace representation. Empirically, this matters for AnaCP, whereas for GACL inter-class mixing
yields higher accuracy (Table~\ref{tab:efficiency}).

\subsection{Mathematical Proof for the Statement} \label{sec:appendix_proof}
\subsubsection{Proof for Theorem \ref{thm:stable_rank}}
We will give the full proof for Theorem \ref{thm:stable_rank}. First we have the following lemma
\begin{lemma}[Extremal eigenvalues of a $2\times 2$ symmetric matrix]
\label{lem:2x2}
Let
\[
K=\begin{pmatrix} a & c\\ c & b\end{pmatrix}
\quad\text{with }a,b\in\mathbb{R},\; c\in\mathbb{R}.
\]
Then
\[
\lambda_{\max}(K)=\frac{a+b+\sqrt{(a-b)^2+4c^2}}{2},
\qquad
\lambda_{\min}(K)=\frac{a+b-\sqrt{(a-b)^2+4c^2}}{2}.
\]
\end{lemma}
\begin{proof}
    The eigenvalues $\lambda$ of $K$ are the roots of the characteristic equation $\det(K - \lambda I) = 0$:
\[
\det \begin{pmatrix} a - \lambda & c \\ c & b - \lambda \end{pmatrix} = 0
\]

Expanding the determinant, we obtain the quadratic equation:
\[
(a - \lambda)(b - \lambda) - c^2 = 0
\]
\[
\lambda^2 - (a + b)\lambda + ab - c^2 = 0
\]

Solving for $\lambda$ gives:
\[
\lambda = \frac{(a + b) \pm \sqrt{[-(a + b)]^2 - 4(1)(ab - c^2)}}{2}
\]
Simplifying the expression under the square root gives:
\[
\lambda = \frac{a + b \pm \sqrt{(a - b)^2 + 4c^2}}{2}
\]
The maximum and minimum eigenvalues are thus:
\[
\lambda_{\max}(K) = \frac{a + b + \sqrt{(a - b)^2 + 4c^2}}{2}
\]
\[
\lambda_{\min}(K) = \frac{a + b - \sqrt{(a - b)^2 + 4c^2}}{2}
\]
\end{proof}
We utilize the following lemma as a fact
\begin{lemma}\label{lem:psd_frob}
If $A\succeq 0$ and $B\succeq 0$, then
\[
\|A+B\|_F^2 \;=\; \|A\|_F^2+\|B\|_F^2 + 2\operatorname{Tr}(AB)
\;\ge\; \|A\|_F^2+\|B\|_F^2.
\]
\end{lemma}
The complete proof for the Theorem \ref{thm:stable_rank} is given by
\begin{proof}
We lower bound the numerator $\|\widetilde G+\tau I\|_F^2$ and upper bound the
denominator $\|\widetilde G+\tau I\|_2^2$. Since $\widetilde G+\tau I$ is symmetric PSD,
\[
\|\widetilde G+\tau I\|_2 = \lambda_{\max}(\widetilde G+\tau I)=\lambda_{\max}(\widetilde G)+\tau.
\]
Thus it suffices to bound $\lambda_{\max}(\widetilde G)=\lambda_{\max}(G + \Delta)$.

Let $x\in\mathbb{R}^d$ with $\|x\|_2=1$, and decompose it as
\[
x = U_h a + U_t b,\qquad a\in\mathbb{R}^r,\; b\in\mathbb{R}^{d-r},
\qquad \|a\|_2^2+\|b\|_2^2=1.
\]
Because $U_h$ spans eigenvectors with eigenvalues in $[\lambda_r,\lambda_1]$
and $U_t$ spans eigenvectors with eigenvalues in $[\lambda_d,\lambda_{r+1}]$,
we have the Rayleigh quotient bound
\[
 x^\top G x = a^\top (U_h^\top G U_h)a + b^\top (U_t^\top G U_t)b
\le \lambda_1\|a\|_2^2 + \lambda_{r+1}\|b\|_2^2.
\]
For the perturbation term,
\[
x^\top \Delta x
= a^\top(U_h^\top \Delta U_h)a + 2a^\top(U_h^\top \Delta U_t)b + b^\top(U_t^\top \Delta U_t)b.
\]
Using Cauchy--Schwarz and the assumed operator norm bounds,
\[
a^\top(U_h^\top \Delta U_h)a \le \varepsilon\|a\|_2^2,\qquad
b^\top(U_t^\top \Delta U_t)b \le \bar m\,\|b\|_2^2,\qquad
2a^\top(U_h^\top \Delta U_t)b \le 2\eta\|a\|_2\|b\|_2.
\]
Therefore,
\[
x^\top(G + \Delta)x
\le (\lambda_1+\varepsilon)\|a\|_2^2 + (\lambda_{r+1}+\bar m)\|b\|_2^2
+2\eta\|a\|_2\|b\|_2.
\]
Define the $2\times 2$ symmetric matrix
\[
K_+ \;\triangleq\;
\begin{pmatrix}
\lambda_1+\varepsilon & \eta\\
\eta & \lambda_{r+1}+\bar m
\end{pmatrix}.
\]
Then the right-hand side equals
$\begin{bmatrix}\|a\|_2\\ \|b\|_2\end{bmatrix}^\top
K_+
\begin{bmatrix}\|a\|_2\\ \|b\|_2\end{bmatrix}$.
Since $\|a\|_2^2+\|b\|_2^2=1$, the maximum over all unit $x$ equals
$\lambda_{\max}(K_+)$. Hence,
\[
\lambda_{\max}(G + \Delta)
=\max_{\|x\|_2=1} x^\top(G + \Delta)x
\le \lambda_{\max}(K_+).
\]
By Lemma~\ref{lem:2x2}, $\lambda_{\max}(K_+)=\Lambda_+$, so
\[
\|\widetilde G+\tau I\|_2
= \lambda_{\max}(G + \Delta)+\tau
\le \Lambda_+ + \tau.
\]
Now for lower bound, by Lemma~\ref{lem:psd_frob},
\[
\|\widetilde G+\tau I\|_F^2
\ge S_\tau + \|\Delta\|_F^2.
\]
We now lower bound $\|\Delta\|_F^2$ using the tail block.
In the orthonormal basis $U=[U_h\;U_t]$, the Frobenius norm is invariant:
\[
\|\Delta\|_F^2 = \|U^\top \Delta U\|_F^2
= \left\|
\begin{pmatrix}
U_h^\top \Delta U_h & U_h^\top \Delta U_t\\
U_t^\top \Delta U_h & U_t^\top \Delta U_t
\end{pmatrix}
\right\|_F^2
\ge \|U_t^\top \Delta U_t\|_F^2.
\]
Let $\{\mu_j\}_{j=1}^{d-r}$ be the eigenvalues of the PSD matrix $U_t^\top \Delta U_t$.
The assumption $\lambda_{\min}(U_t^\top \Delta U_t)\ge m$ implies $\mu_j\ge m$ for all $j$,
so
\[
\|U_t^\top \Delta U_t\|_F^2
=\sum_{j=1}^{d-r}\mu_j^2
\ge (d-r)m^2.
\]
Therefore,
\[
\|\widetilde G+\tau I\|_F^2
\ge S_\tau + (d-r)m^2.
\]
Combine the bounds to get stable rank
\[
\text{sr}(\widetilde G+\tau I)
=\frac{\|\widetilde G+\tau I\|_F^2}{\|\widetilde G+\tau I\|_2^2}
\ge
\frac{S_\tau + (d-r)m^2}{(\Lambda_+ + \tau)^2},
\]
which proves \eqref{eq:sr_lower_bound}. Since $\text{sr}(G+\tau I)=\dfrac{S_\tau}{(\lambda_1+\tau)^2}$, a sufficient condition for
$\text{sr}(\widetilde G+\tau I)>\text{sr}(G+\tau I)$ is
\[
\frac{S_\tau + (d-r)m^2}{(\Lambda_+ + \tau)^2}
>
\frac{S_\tau}{(\lambda_1+\tau)^2},
\]
which is equivalent (after cross-multiplying positive denominators) to
\eqref{eq:sr_strict_condition_general}. This completes the proof.
\end{proof}

We also acknowledge from a theoretical viewpoint, however, a higher stable rank does not necessarily lead to better generalization performance. An increase in stable rank means that the sum of the remaining singular/eigenvalues becomes larger relative to $\lambda_1$, but it does not necessarily imply that all the remaining $\lambda_i$'s increase uniformly. In any case, stable rank is still a useful metric for indicating that our method has the potential to generalize better, although it should not be interpreted as an if-and-only-if condition. We believe that it is promising to further demonstrate the improvement of our method using other metrics, such as Spectral entropy/Tail energy ratio..., and we leave these results for future work.

% --- Optional short corollary: the clean "tail-only injection" case ------------

\section{Experimental setup in details}
\label{exp_detail}
\subsection{Datasets and Protocols.} We evaluate our method on 4 common benchmark datasets widely used in Class-Incremental Learning (CIL), but in Long-Tailed setting: \textit{CIFAR-100}, \textit{ImageNet-R}, \textit{Tiny-ImageNet}, and \textit{CUB-200}. To simulate realistic Long-Tailed CIL scenarios, we adopt a shuffled task assignment in the spirit of \cite{liu2022long}, where head and tail classes co-occur within each task. Particularly, 
% \begin{itemize}
%     \item \textit{Class Splitting:} For each original dataset, we divide classes into $T=10$ disjoint tasks, where classes arrive sequentially.
%     \item \textit{Long-Tailed Imbalance:} We construct long-tailed variants by exponentially decaying the number of training samples per class. The imbalance ratio $\rho = N_{\max} / N_{\min}$ is modified based on the number of samples/class in each dataset, creating a severe "tail" where minority classes have as few as 5 samples for the main experiments.
% \end{itemize}
\begin{itemize}
  \item[--] \textbf{Class Splitting:} For each original dataset, we randomly
  shuffle the classes and divide them into $T = 10$ disjoint tasks that arrive
  sequentially; the shuffling ensures that every task contains a mixture of
  head and tail classes.
  \item[--] \textbf{Imbalance Construction:} We use a two-level (step)
  imbalance~\cite{cao2019learning,DBLP:conf/cvpr/CuiJLSB19}: a fraction of classes (\emph{head};
  $30\%$ by default) retain their full training set, while the remaining
  (\emph{tail}) classes are subsampled to a small fixed shot count $N_{\min}$
  ($5$ samples in the main experiments). The imbalance ratio is defined as
  $\rho = N_{\max}/N_{\min}$, where $N_{\max}$ is the head-class sample count;
  we vary $\rho$ by changing the tail shot count $N_{\min}$.
\end{itemize}

\subsection{Baselines.} To comprehensively evaluate the effectiveness of GSR, we compare it against a wide range of representative methods, categorized into two groups:

% \textcolor{red}{[NEED TO CHECK AND REVISE]}

\begin{itemize}
    \item \textit{Analytic Class-Incremental Learning (State-of-the-Arts):} This is our main comparison group, representing the state-of-the-art in analytic learning. We evaluate our method against:
    
    \begin{itemize}
        \item \textit{GACL} \cite{GACL_Zhuang_NeurIPS2024}: A generalized analytic continual learning framework that utilizes standard Ridge Regression. This serves as the fundamental baseline to demonstrate the vulnerability of standard isotropic regularization under spectral collapse.
        \item \textit{RanPAC} \cite{mcdonnell2023ranpac}: A leading method that employs random projections to expand the feature space before applying RLS, aiming for better feature decorrelation.
        \item \textit{AnaCP} \cite{DBLP:journals/corr/abs-2511-13880}: The current state-of-the-art in RLS-based ACL, which utilizes analytic contrastive projection to encourage well-separated clusters.
        \item \textit{AIR} \cite{AIR_Fang_arXiv2024}: Analytic Imbalance Rectifier. This is a crucial baseline as it is the only existing RLS-based method explicitly designed for long-tailed settings via inverse-frequency re-weighting. Comparing with AIR contrasts our rectification with inverse-frequency re-weighting.
    \end{itemize}
    
    % \textbf{C-FSL} [Ref] and \textbf{ALICE} [Ref]. More importantly, we benchmark against the current leading methods: \textbf{RanPAC} [10], which utilizes random projections for feature decorrelation, and \textbf{AnaCP} [11], which focuses on analytic covariance profiling.
    
    % \item \textit{Traditional Optimization-based CIL:} To demonstrate the efficiency gap, we include classic gradient-based methods such as \textbf{EWC} [Ref] and \textbf{LwF} [Ref]. Following the standard analytic learning protocol, we adapt these methods to the fixed backbone setting, where only the classifier heads are fine-tuned.
    
    \item \textit{Covariance and Prototype-based Methods:} Since our method operates on feature covariance geometry and prototype representations, we benchmark against established statistical and simple approaches:
    \begin{itemize}
        \item \textit{SimpleCIL} \cite{DBLP:journals/corr/abs-2210-04428}: A baseline that directly uses frozen PTM features with the Nearest Class Mean classification strategy. As shown in our analysis, it often outperforms complex RLS methods in long-tailed scenarios.
        \item \textit{SLDA} \cite{Hayes_2020_CVPR_Workshops}\& \textit{KLDA} \cite{momeni2025continual}: Streaming and Kernelized Linear Discriminant Analysis, which maintain running means and a shared global covariance matrix for incremental classification.
        \item \textit{FeCAM} \cite{goswami2023fecam}: Feature Covariance Alignment. It exploits the heterogeneity of class distributions by estimating class-specific covariance matrices. 
        % This comparison demonstrates the advantage of our Spherical Mixup strategy over standard empirical covariance estimation in few-shot/tail regimes.
    \end{itemize}
    
    % Since GSR operates on feature covariance geometry, we also adapt \textbf{FeCAM} [Ref] and \textbf{SLCA} [Ref]—originally designed for few-shot or base CIL tasks—to our Long-Tailed Incremental setting. This comparison highlights the superiority of our \textit{Spherical Rectification} over standard covariance estimation techniques.
\end{itemize}

\subsection{Implementation Details.}
\begin{itemize}
    \item \textbf{Backbone:} Consistent with recent SOTA analytic methods \cite{DBLP:journals/corr/abs-2511-13880, mcdonnell2023ranpac}, we utilize frozen Pre-Trained Models (PTMs) to extract features. Specifically, we report results using \textit{DINO-v2} and \textit{MoCo-v3} to demonstrate robustness across different feature spaces. 
    \item \textbf{Hyperparameters:} All RLS-based methods operate on a random-projection buffer layer of dimension $5000$ with ridge parameter $\tau = 10^{2}$ for both backbones. For GSR, we use $\gamma \sim \mathrm{Beta}(2, 2)$ and $T_{\max} = 200$. With GACL and RanPAC we use $\beta = 1$. AnaCP and AnaCP + Ours use three contrastive-projection layers and $100$ replay samples per class; AnaCP + Ours uses class-specific covariances for the pseudo-replay with DINO-v2 and a shared covariance with MoCo-v3.

    \noindent\textbf{Hyperparameter selection for AnaCP + Ours.}
The AnaCP variant of our method has three hyperparameters of its own: the
rectification weight $\beta$, the exponent $p$ of the count-based logit
re-weighting $N_c^{-p}$, and the over-sampling factor $r$ of tail classes in
the pseudo-replay. For each backbone and dataset we select them from
$\beta \in \{0.1, 0.25, 0.5, 1\}$, $p \in \{0, 0.25, 0.5\}$ and
$r \in \{1, 4\}$ on a validation set of $20\%$ of the training images of every
class, held out before the long-tailed split is built; the test set is not
used for selection. The selected configuration (Table~\ref{tab:selected}) is
then trained on the full training set and evaluated once. The ridge parameter
$\tau = 10^2$ and the projection dimension are shared with all other RLS-based
methods and are not tuned. All baselines, including AnaCP, use their default
configuration.

\begin{table}[h]
\centering
\caption{Hyperparameters of AnaCP + Ours selected on the validation set.}
\label{tab:selected}
\begin{tabular}{llccc}
\toprule
Backbone & Dataset & $\beta$ & $p$ & $r$ \\
\midrule
DINO-v2 & CIFAR-100     & 0.25 & 0.5  & 1 \\
        & ImageNet-R    & 0.1  & 0.5  & 1 \\
        & Tiny-ImageNet & 0.5  & 0.5  & 1 \\
        & CUB-200       & 0.5  & 0.5  & 1 \\
\midrule
MoCo-v3 & CIFAR-100     & 0.1  & 0.25 & 4 \\
        & ImageNet-R    & 0.1  & 0.25 & 1 \\
        & Tiny-ImageNet & 0.25 & 0.25 & 4 \\
        & CUB-200       & 0.1  & 0.5  & 4 \\
\bottomrule
\end{tabular}
\end{table}
    
    % For GSR, we use $\gamma \sim \text{Beta}(2, 2)$, $T_{\max} = 200$, and $\beta = 1$ for simplicity.
    % For GSR, the base mixing intensity $\alpha_{base}$ (Eq. 8) is set to 0.6, and the decay rate $\xi$ is set to $5 \times 10^{-3}$. The ridge regularization parameter $\tau$ is set to $\tau = 0.01$ for DINO-v2, and $\tau=0.001$ for MoCo-v3 backbone, $\beta$ is set to 1 for simplicity.
    % as our spectral rectification inherently stabilizes the inversion.
    % \item \textbf{Baselines:} We compare GSR against:
    % (1) \textit{Traditional CIL:} EWC, LwF (adapted for fixed backbones);
    % (2) \textit{Analytic CIL:} C-FSL, ALICE, and the current SOTA methods \textbf{RanPAC} [10] and \textbf{AnaCP} [11].
    % (3) \textit{Covariance-based methods:} We also adapt \textbf{FeCAM} and \textbf{SLCA} to the incremental setting as strong baselines for covariance estimation.
\end{itemize}

\subsection{Evaluation Metrics.} We report the performance using:
\begin{itemize}
    \item \textit{Last Accuracy ($A_{last}$):} The overall test accuracy after learning the final task $T$.
    \item \textit{Average Accuracy ($A_{avg}$):} The mean of test accuracies calculated after each task step, reflecting the stability of the learning curve.
    % \item \textbf{Forgetting ($F$):} The average performance drop of old classes after learning new tasks.
\end{itemize}
All experiments are averaged over 3 random seeds to ensure statistical significance.

\section{Additional experimental results}
\label{additional_exp}

 Figure \ref{fig:standard_longtail_full} illustrates a comprehensive performance comparison of various ACL baselines under both standard (balanced) and long-tailed CIL settings. The results reveal a catastrophic performance drop in the long-tailed regime, most notably among state-of-the-art RLS-based methods such as AnaCP, RanPAC, and GACL. Intriguingly, this trend is highly consistent: simplistic baselines like SimpleCIL—which rely merely on a basic Nearest Class Mean strategy—exhibit remarkable resilience and surprisingly outperform these sophisticated SOTA methods under severe imbalance. This stark contrast underscores the critical need to diagnose the root cause of this vulnerability and to develop robust solutions for RLS-based methods, ensuring they can maintain their upper-bound performance even in challenging settings of long-tailed scenarios.

% 1. Full exp longtailed vs standard. Figure \ref{fig:standard_longtail_full}.

\begin{figure}
    \centering
    \includegraphics[width=\linewidth]{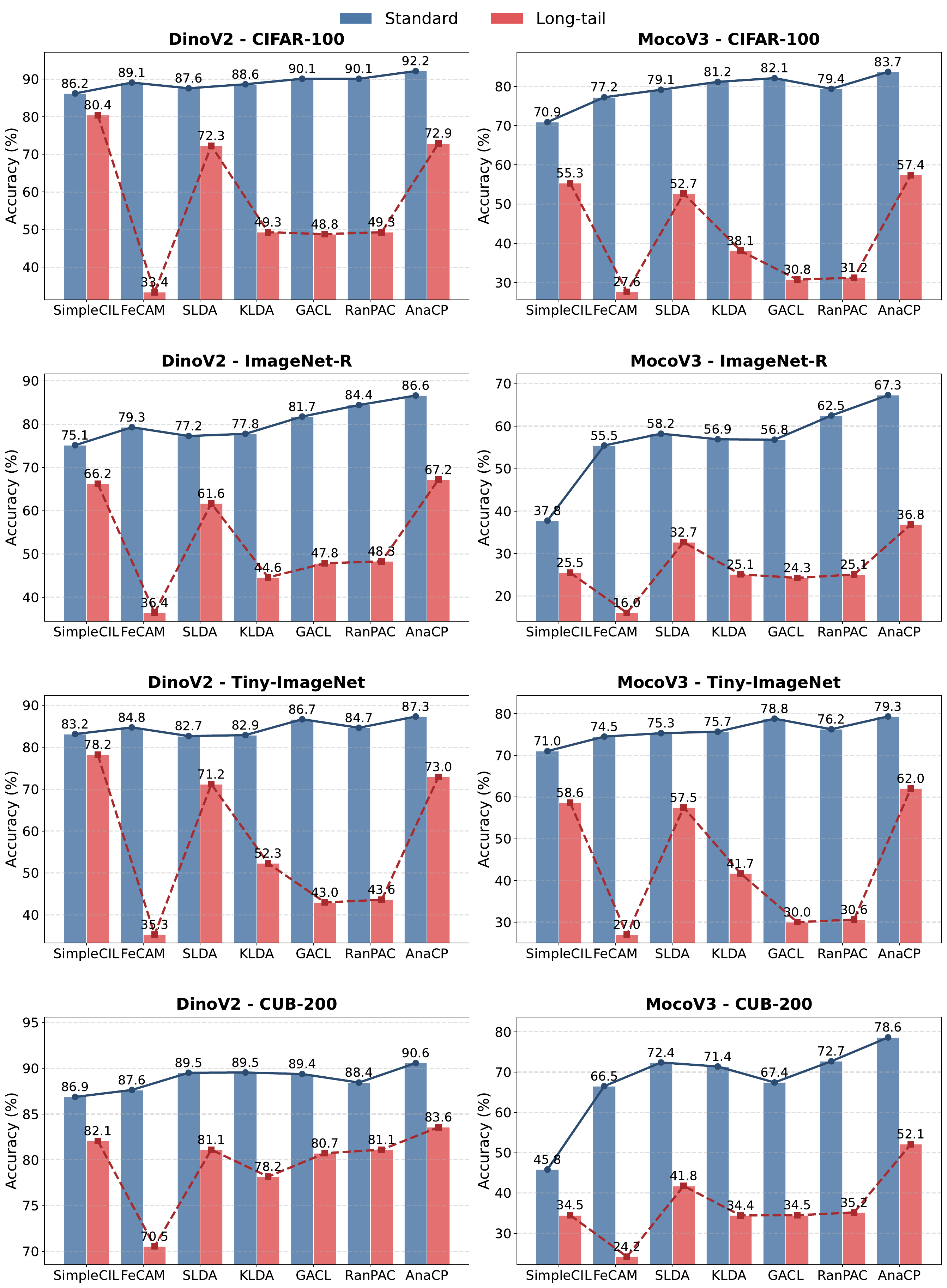}
    \caption{Performance of baselines in standard vs long-tailed settings - $A_{last}$ (\%) on different datasets.}
    \label{fig:standard_longtail_full}
\end{figure}

\section{Author Note: Changes from the proceedings version}
(i) The class-adaptive component is now described as the per-class
rectification budget $M_c$ (Eq.~8), which is what the released implementation
uses. (ii) Hyperparameters in App.~C.3 now match \href{https://github.com/tranquyenbk173/GSR_Geometry-Spectral-Rectification}{the released code.}. (iii)
Tables~3 and~6 have been re-run with the released code; in Table~6 the entry
for inter-class mixup on GACL changed substantially. (iv) Two AIR entries and
one AnaCP entry in Table~2 were corrected to match our experiment logs. (v) Fig.~2
now reports the stable rank of the mean-centred per-task statistics. The claims about the
stable rank in the abstract and Sections~1 and~3 were revised accordingly.

\end{document}